\documentclass[journal]{IEEEtran}

\ifCLASSINFOpdf
\else
   \usepackage[dvips]{graphicx}
\fi
\usepackage{url}

\usepackage{graphicx}
\usepackage{multirow}
\usepackage{booktabs}
\usepackage{amsfonts}
\usepackage{amsopn}
\usepackage{amsmath,amsthm}
\usepackage{amssymb,amsfonts}
\usepackage{algorithm}
\usepackage{algorithmic}
\usepackage{tabularx}
\usepackage{epstopdf}
\usepackage{epsfig}
\usepackage{graphicx}
\usepackage{float}
\usepackage{subfig}
\usepackage{makecell}
\usepackage{float}
\graphicspath{{image/}}
\usepackage{color,xcolor}
\usepackage[numbers,sort&compress]{natbib}% our
\usepackage{footnote}
\usepackage{url}
\usepackage{multirow}
\usepackage{booktabs}
\newtheorem{lemma}{Lemma}
\newtheorem{remark}{Remark}%[section]
\newtheorem{theorem}{Theorem}%[section]
\usepackage{soul}% Ad\textcolor{red}{denoising }d by us
\usepackage{color}

\begin{document}

\title{Hyperspectral Image Denoising via Multi-modal and Double-weighted Tensor Nuclear Norm}

\author{Sheng Liu, Xiaozhen Xie, and Wenfeng Kong
%\thanks{%Manuscript received xx xx, xx. revised xx xx, xx; accepted xx xx, xx. Date of publication xx xx, xx; date of current version xx xx, xx. 
%            This work was supported by the Undergraduate Innovation and Entrepreneurship Project of China under Grant S202010712032, S202010712294 and X202010712196. (Corresponding author: Xiaozhen Xie.)}
\thanks{Sheng Liu, Xiaozhen Xie, and Wenfeng Kong are with College of Science, Northwest A\&F University, Yangling 712100, China (e-mail: liu\_sheng@nwafu.edu.cn; e-mail: xiexzh@nwafu.edu.cn (Corresponding author); e-mail: kksama@nwafu.edu.cn).}}

\markboth{Journal of \LaTeX\ Class Files, Vol. 14, No. 8, August 2015}
{Shell \MakeLowercase{\textit{et al.}}: Bare Demo of IEEEtran.cls for IEEE Journals}
\maketitle

\begin{abstract}

Hyperspectral images (HSIs) usually suffer from different types of pollution.
This severely reduces the quality of HSIs and limits the accuracy of subsequent processing tasks.
HSI denoising can be modeled as a low-rank tensor denoising problem.
Tensor nuclear norm (TNN) induced by tensor singular value decomposition plays an important role in this problem.
In this letter, we first reconsider three inconspicuous but crucial phenomenons in TNN.
In the Fourier transform domain of HSIs, different frequency slices (FS) contain different information;
different singular values (SVs) of each FS also represent different information.
The two physical phenomenons lie not only in the spectral mode but also in the spatial modes.
Then based on them, we propose a multi-modal and double-weighted TNN.
It can adaptively shrink the FS and SVs according to their physical meanings in all modes of HSIs.
In the framework of the alternating direction method of multipliers, we design an effective alternating iterative strategy to optimize our proposed model.
Denoised experiments on both synthetic and real HSI datasets demonstrate their superiority against related methods.

\end{abstract}


\IEEEpeerreviewmaketitle

%\vspace{-0.6cm}
\section{Introduction}
%\IEEEPARstart{H}{yperspectral} image (HSI) has been widely used in many fields \cite{r1goetz2009three, r2bioucas2013hyperspectral} due to its wealthy spatial and spectral information of a real scene.
\IEEEPARstart{H}{yperspectral} image (HSI) has been widely used in many fields \cite{ r2bioucas2013hyperspectral} due to its wealthy spatial and spectral information of a real scene.
However, the observed HSIs are usually corrupted by different noises, e.g., Gaussian noise, impulse noise, deadlines, stripes and their mixtures.
This seriously affects the subsequent applications of HSI, such as unmixing, target detection, and so on. 
Therefore, HSI denoising , as a preprocessing step to remove mixed noise for various subsequent applications, is a valuable and active research topic.

In HSIs, different spectral bands are images from the same scene under different wavelengths.
It means the global correlation or low-rank prior lies in HSIs \cite{Denoising_GRSL}.
Based on this prior, how to measure the HSI low-rankness becomes the key to denoising tasks.
As HSIs can be treated as 3rd-order tensor,   this problem is turned into a 3rd-order tensor low-rank problem. 
Due to the nonunique definitions of the tensor rank, different tensor decompositions and their corresponding tensor ranks are proposed,
such as the Tucker decomposition \cite{LRTDTV,Tuckear_GRSL},  PARAFAC decomposition \cite{CP_r1}, and tensor singular value decomposition (t-SVD) \cite{lu2019TRPCA,3DTNN}, to exploit the low-rankness of HSIs.

Among them, the tensor tubal rank induced by t-SVD can characterize the low-rank structure of HSIs very well.% \cite{zhang2014novel_r21}.
Its convex relaxation is the tensor nuclear norm (TNN) \cite{t_SVD}.
TNN is effective to keep the intrinsic structure of HSIs.
%t-SVD can be calculated easily in the Fourier domain and TNN minimization problem can be efficiently solved by convex optimization algorithms. 
Hence, TNN has attracted extensive attention for HSI denoising problems in recent years \cite{3DTNN,zengTGRS,zengSP}.
However, during the definition of TNN, there are three kinds of prior knowledge that are underutilized for further exploiting the low-rankness in HSIs.
Firstly, in the Fourier transform domain of HSIs,
the low-frequency slices carry the profile information of HSIs,
while the high-frequency slices mainly carry the noise information of HSIs.
Secondly, in each frequency slices,
bigger singular values mainly contain information on clean data and smaller singular values mainly contain information on noise.
Thirdly, low-rankness not only exists in the spectral dimension but also lies in the spatial dimensions \cite{LRTDTV}.
The classical TNN only takes the Fourier transform to connect the spatial dimensions with the spectral dimension
and lacks flexibility for handling different correlations along with different modes of HSIs \cite{3DTNN}.

In this letter, to take full advantage of the above prior knowledge and improve the capability and flexibility of TNN,
we propose a multi-modal and double-weighted TNN for HSI denoising tasks.
The merits of our model are four-fold.
First, according to information types in different frequency slices in the Fourier transform domain,
we adaptively assign bigger weights to slices that mainly contain noise information and smaller weights to slices that mainly contain profile information,
which can depress noise more and simultaneously preserve the profile information of clean HSIs better.
Second, in each frequency slice,
we use the partial sum of singular values to only shrink small singular values, which can better protect the clean data information contained in big singular values.
Third, we apply the double-weighted TNN in all modes of HSIs, which can achieve a more flexible and accurate characterization of HSI low-rankness.
Finally, we develop an alternating direction method of multiplier based algorithm to efficiently solve the proposed model,
Compared with various competing HSI denoising methods, the best-denoised performances are obtained by our method synthetic and real HSI datasets.

%\vspace{-0.6cm}
\section{Preliminaries}
\vspace{-0.3cm}
\subsection{Notations}
In this letter, matrix and tensor are denoted as bold upper-case letter $\mathbf{X}$ and calligraphic letter $\mathcal{X}$, respectively.
For a 3rd-order tensor $\mathcal{X}\in\mathbb{R}^{n_1 \times n_2 \times n_3}$, its $(i,j,k)$-th component is represented as $\mathcal{X}(i,j,k)$.
For $\mathcal{X}$, $\mathcal{Y} \in\mathbb{R}^{n_1 \times n_2 \times n_3}$, their inner product is defined as $<\mathcal{X},\mathcal{Y}>=\sum_{i=1}^{n_1}\sum_{j=1}^{n_2}\sum_{k=1}^{n_3}x_{ijk}y_{ijk}$.
Then the Frobenius norm of a tensor $\mathcal{X}$ is defined as $\|\mathcal{X}\|_{F}=\sqrt{<\mathcal{X},\mathcal{X}>}$.
The $k$-th frontal slice of $\mathcal{X}$ is denoted as $\mathbf{X}^{(k)}=\mathcal{X}(:, :,k )$.
The fast Fourier transform along the third mode of $\mathcal{X}$ is represented as $\bar{\mathcal{X}}=\texttt{fft}(\mathcal{X},[],3)$
and its inverse operation is $\mathcal{X}=\texttt{ifft}(\bar{\mathcal{X}},[],3)$.
The mode-$p$ permutation of $\mathcal{X}$ is defined as $\mathcal{X}_{p}=\texttt{permute}(\mathcal{X},p)$, $p=1,2,3$,
where the $m$-th mode-3 slice of $\mathcal{X}_{p}$ is the $m$-th mode-$p$ slice of $\mathcal{X}$, i.e.,
$\mathcal{X}(i,j,k)=$$\mathcal{X}_{1}(j,k,i)=$$\mathcal{X}_{2}(k,i,j)=$$\mathcal{X}_{3}(i,j,k)$.
Also, its inverse operation is  $\mathcal{X}=\texttt{ipermute}(\mathcal{X}_{p},p)$.

\subsection{Problem Formulation}
An ideal HSI can be viewed as a 3rd-order tensor $\mathcal{X}\in\mathbb{R}^{n_1 \times n_2 \times n_3}$
and usually is assumed to be low-rank.
Corrupted by mixed noise, its observed version can be modeled as
\begin{equation}
\label{eq_1}
  \mathcal{Y} = \mathcal{X} + \mathcal{S} + \mathcal{N},
\end{equation}
where $\mathcal{Y}, \mathcal{S}, \mathcal{N} \in \mathbb{R}^{n_1 \times n_2 \times n_3}$;
$\mathcal{S}$ denotes the sparse noise;
$\mathcal{N}$ denotes the Gaussian white noise.

HSI denoising aims to recover the ideal HSI $\mathcal{X}$ from the observed HSI $\mathcal{Y}$ in (\ref{eq_1}).
Under the framework of regularization theory, it can briefly be formulated as
\begin{equation}
\label{eq_2}
\begin{aligned}
\arg\min_{\mathcal{X}, \mathcal{S}, \mathcal{N}} & \texttt{Rank}(\mathcal{X}) + \lambda\|\mathcal{S}\|_1 + \tau\|\mathcal{N}\|_F^2, \\
s.t. & \mathcal{Y} = \mathcal{X} + \mathcal{S} + \mathcal{N},
\end{aligned}
\end{equation}
where $\|\cdot\|_1$ is $L_1$ norm to detect the sparse noise;
$\|\cdot\|_F$ describes the Gaussian noise;
$\texttt{Rank}(\cdot)$ represents the rank of unknown ideal HSI;
$\lambda$ and $\tau$  are non-negative parameters.

In model (\ref{eq_2}),  regularization term $\texttt{Rank}$ is approximated by different relaxations. % in various HSI restoration model.
As mentioned above, TNN is widely used convex relaxation, which can be defined as
\begin{equation}
\label{eq_3}
%\begin{array}{l}
  \|\mathcal{X}\|_* := \frac{1}{n_3}\sum_{k=1}^{n_3}\|\bar{\textbf{X}}^{(k)}\|_*.
%\end{array}
\end{equation}

\section{The Weighted TNN}

\subsection{Frequency-Weighted TNN}
In \eqref{eq_3}, the frontal slice of $\bar{\mathcal{X}}$ corresponds to the frequency component  of $\mathcal{X}$ \cite{wang2020frequency}. 
Specifically, for $\mathcal{X}$, its profile information is contained in the low-frequency frontal slices,
while its detailed information is contained in the high-frequency ones.
When $\mathcal{X}$ is distorted by outliers, the effects on high-frequency frontal slices are more severe.
However, different frequency slices of $\bar{\mathcal{X}}$ have the same impact on TNN in \eqref{eq_3},
which is obviously inconsistent with the physics meaning of frequency components.
Therefore, we improve TNN  in \eqref{eq_3} by assigning different weights for different frequency slices, and propose the frequency-weighted TNN as follows:
\begin{equation}
\label{eq_4}
%\begin{array}{l}
  \|\mathcal{X}\|_{w*} := \sum_{k=1}^{n_3}w_{k}(\bar{\textbf{X}}^{(k)})\|\bar{\textbf{X}}^{(k)}\|_*,
%\end{array}
\end{equation}
where $w_{k}(\bar{\textbf{X}}^{(k)})$ is the $k$-th weight parameter. 
For HSI denoising problems, the lower the frequencies are, the less the corresponding frequency slices should be punished.
By amounts of data simulations, the weights $w_{k}$ approximatively consist with the frequencies and are inversely proportionate to $\|\bar{\textbf{X}}^{(k)}\|_F$. 
We let
\begin{equation}
\label{eq_5}
  w_{k}(\bar{\textbf{X}}^{(k)}) = \frac{c_1}{\log(\|\bar{\textbf{X}}^{(k)}\|_F^2)+\varepsilon} +c_2,
\end{equation}
where $\varepsilon=10^{-10}$ is to avoid dividing by zero; 
$c_1$ is the scaling factor after frequency normalization; $c_2$ is a constant.
%$c_1$ and $c_2$ are two parameters.
The proposed FWTNN is different from the frequency-filtered tensor nuclear norm (FTNN) \cite{wang2020frequency}.  Our weights are  data dependent, but FTNN's weights  are pre-weight.  
%\vspace{-0.3cm}
\subsection{Double-Weighted TNN} 

For $\bar{\textbf{X}}^{(k)}$ in \eqref{eq_3},  the matrix nuclear norm is used as the tightest convex surrogate for rank.
However, it has limitation in the accuracy of approximation due to its convexity.
Recently, a series of improvement methods are proposed for better approximation \cite{PSSV, WNNM}.
To differently treat singular values of $\bar{\textbf{X}}^{(k)}$, we choose partial sum of singular values (PSSV) to only punish the smaller singular values which mainly contain the noise information of HSIs.
Then, a double-weighted TNN is proposed by replacing the matrix nuclear norm in (\ref{eq_4}) with the PSSV of $\bar{\textbf{X}}^{(k)}$, which is defined as
\begin{equation}
\label{eq_6}
%\begin{array}{l}
  \|\mathcal{X}\|_{dw*} := \frac{1}{n_3}\sum_{k=1}^{n_3}w_{k}(\bar{\textbf{X}}^{(k)})\|\bar{\textbf{X}}^{(k)}\|_{\textrm{PSSV}},
%\end{array}
\end{equation}
where $\|\bar{\textbf{X}}^{(k)}\|_{\textrm{PSSV}}= \sum_{r=R+1}^{\min\{ n_1,n_2 \}} \sigma_r(\bar{\textbf{X}}^{(k)})$;
$\sigma_r(\bar{\textbf{X}}^{(k)})$ is the $r$-th biggest singular value of matrix $\bar{\textbf{X}}^{(k)}$;
$R$ is a parameter indicating the number of main singular values. The double-weighted TNN minimization problem can be solved by following theorem.

\begin{theorem} \label{DWSVT_theorem}
	Assuming that $\tau >0$, $\mathcal{X},\mathcal{Y}\in$ $\mathbb{R}^{n_1\times n_2\times n_3}$,
	for the minimization problem
	\begin{equation} \label{dwtnn_question}
	\mathcal{X}^*=\arg\min_{\mathcal{X}}\tau \lVert \mathcal{X} \rVert _{dw*}+\frac{1}{2}\lVert \mathcal{X}-\mathcal{Y} \rVert _{F}^{2},
	\end{equation}
	its solution is
	\begin{equation}
	\begin{array}{rl}
	\mathcal{X}^*=\mathcal{D}\mathcal{W}^{w,R,\tau}\left( \mathcal{Y} \right) = \texttt{ifft}(\bar{\mathcal{U}} \cdot \bar{\mathcal{S}}_{dw*} \cdot \bar{\mathcal{V}}^T, [], 3),
	\end{array}
	\end{equation}
	where $\bar{\mathcal{Y}}=\bar{\mathcal{U}} \cdot \bar{\mathcal{S}} \cdot \bar{\mathcal{V}}^T$;
	$
	\bar{\mathcal{S}}_{dw*}(r,r,k)= \max ( \bar{\mathcal{S}}(r,r,k)-\tau w_rw_k  ,0 );
	$, $R=TW(k)$
	$w_r =\left\{ \begin{array}{l}
	\begin{matrix}
	0,&		r\le R\\
	\end{matrix}\\
	\begin{matrix}
	1,&		r>R\\
	\end{matrix}\\
	\end{array} \right. $;
	$w_k = \frac{c_1}{\log ( \lVert \boldsymbol{\bar{\textbf{X}}}^{( k )} \rVert _F^2+\varepsilon ) } +c_2$. 
	\end{theorem}
%\noindent {\textbf{Theorem 1.} }

\begin{lemma}(PSVT \cite{PSSV}). \label{PSVT_lemma}
	  Let  $\tau>0$, $l=\min (m, n)$  and  $\mathbf{X}, \mathbf{Y} \in   \mathbb{R}^{m \times n} $ which can be decomposed by SVD. $\mathbf{Y}$ can be considered as the sum of two matrices,  $\mathbf{Y}=\mathbf{Y}_{1}+\mathbf{Y}_{2}=\mathbf{U}_{Y 1} \mathbf{D}_{Y 1} \mathbf{V}_{Y 1}^{\top}+   \mathbf{U}_{Y 2} \mathbf{D}_{Y 2} \mathbf{V}_{Y 2}^{\top}$ , where  $\mathbf{U}_{Y 1}, \mathbf{V}_{Y 1}$  are the singular vector matrices corresponding to the  $R$  largest singular values by  S V D , and  $\mathbf{U}_{Y 2}$, $\mathbf{V}_{Y 2}$  from the  $(R+1)$ -th to the last singular values. Define a minimization problem for the PSSV as
	\begin{equation}\label{PSVT_lemma02}
	\underset{\mathbf{X}}{\arg \min } \frac{1}{2}\|\mathbf{X}-\mathbf{Y}\|_{F}^{2}+\tau\|\mathbf{X}\|_{\textrm{PSSV}}
	\end{equation}
	Then, the optimal solution of Eq.\eqref{PSVT_lemma02} can be expressed by the PSVT operator defined as:
	\begin{equation}\label{PSVT_lemma_01}
	\begin{aligned}
	\mathbb{P}_{R, \tau}[\mathbf{Y}] &=\mathbf{U}_{Y}\left(\mathbf{D}_{Y 1}+\mathcal{S}_{\tau}\left[\mathbf{D}_{Y 2}\right]\right) \mathbf{V}_{Y}^{\top} \\
	&=\mathbf{Y}_{1}+\mathbf{U}_{Y 2} \mathcal{S}_{\tau}\left[\mathbf{D}_{Y 2}\right] \mathbf{V}_{Y 2}^{\top}
	\end{aligned}
	\end{equation}
	where 
	\begin{equation}
	\begin{array}{l} 
	\mathbf{D}_{Y 1}=\operatorname{diag}\left(\sigma_{1}, \cdots, \sigma_{R}, 0, \cdots, 0\right) \\
	\mathbf{D}_{Y 2}=\operatorname{diag}\left(0, \cdots, 0, \sigma_{R+1}, \cdots, \sigma_{l}\right)
	\end{array}
	\end{equation}
	and  $\mathcal{S}_{\tau}[x]= \max (|x|-\tau, 0)$  is the soft-thresholding operator \cite{L1_Solve}. 
\end{lemma}

\begin{proof}
	Let $\tau>0$, $\mathcal{X},\mathcal{Y}\in$ $\mathbb{R}^{n_1\times n_2\times n_3}$, $MR\in\mathbb{R}^{n_3} $ is  multi-rank of $\mathcal{X}$, and $c_1, c_2$ are given constants. 
	According to the definition of DWTNN \eqref{eq_6} and the properties of tensors in the Fourier domain, the Eq. \eqref{dwtnn_question} is equivalent to
	\begin{equation}\label{DWTNN_proof01}
	\begin{aligned}
	&\mathcal{X}^{*} =\arg \min _{\mathcal{X}} \tau\|X\|_{d w *}+\frac{1}{2}\|\mathcal{X}-\mathcal{Y}\|_{F}^{2}, \\
	&=\arg \min _{\mathcal{X}} \tau \frac{1}{n_{3}} \sum_{k=1}^{n_{3}} w_{k}\left(\overline{\mathbf{X}}^{(k)}\right)\left\|\overline{\mathbf{X}}^{(k)}\right\|_\textrm{PSSV}+\frac{1}{2 n_{3}}\|\overline{\mathcal{X}}-\overline{\mathcal{Y}}\|_{F}^{2} \\
	&=\arg \min _{\mathcal{X}} \frac{1}{n_{3}} \sum_{k=1}^{n_{3}} \tau w_{k}\left(\overline{\mathbf{X}}^{(k)}\right)\left\|\overline{\mathbf{X}}^{(k)}\right\|_{\textrm{PSSV}}+\frac{1}{2}\left\|\overline{\mathbf{X}}^{(k)}-\overline{\mathbf{Y}}^{(k)}\right\|_{F}^{2} .
	\end{aligned}
	\end{equation}
	Therefore, the Eq. \eqref{DWTNN_proof01} can be divided into $n_3$ subproblems as follows:
	\begin{equation}\label{DWTNN_proof02}
	\arg \min _{\mathcal{X}} \tau w_{k}\left(\overline{\mathbf{X}}^{(k)}\right)\|\overline{\mathbf{X}}\|_{p=R}+\frac{1}{2}\|\overline{\mathbf{X}}-\overline{\mathbf{Y}}\|_{F}^{2},
	\end{equation}
	the Eq.\eqref{DWTNN_proof02} can be solved by Lemma \ref{PSVT_lemma}. 
	Let 
	\begin{equation}\label{DWTNN_proof03}
	w_{r}=\left\{\begin{array}{ll}
	0, & i \leqslant R \\
	1, & i>R
	\end{array}\right.
	\end{equation}
	Base on Eq.\eqref{DWTNN_proof03}, Eq.\eqref{PSVT_lemma_01} in Lemma \ref{PSVT_lemma} is equivalent to
	\begin{equation}
	\begin{aligned}
	\mathbb{P}_{R,\tau, w_k }\left[ \overline{\mathbf{Y} }\right] =\overline{\mathbf{U}}_Y\left( \mathcal{S}_{R,\tau, w_k }\left[\overline{ \mathbf{D} }\right] \right) \overline{\mathbf{V}}_{Y}^{\top}
	\end{aligned}
	\end{equation}
	where 
	\begin{equation}
	\begin{aligned}
	\mathcal{S}_{R,\tau, w_k }\left[ \overline{\mathbf{D}}_i \right] 
	&=\begin{cases}
	\sigma _i\left( \overline{\mathbf{Y}} \right)&		\text{if}\  i<R+1\\
	\max \left( \sigma _i\left( \overline{\mathbf{Y}} \right) -\tau w_k,0 \right)&		\text{otherwise}\\
	\end{cases}\\
	&=\max \left( \sigma _i\left( \overline{\mathbf{Y}} \right) -\tau w_kw_r,0 \right) \\
	\end{aligned}
	\end{equation}
	Therefore, the solution of Eq. \eqref{dwtnn_question} is
	\begin{equation}
	\begin{array}{rl}
	\mathcal{X}^*=\mathcal{D}\mathcal{W}^{w,R,\tau}\left( \mathcal{Y} \right) = \texttt{ifft}(\bar{\mathcal{U}} \cdot \bar{\mathcal{S}}_{dw*} \cdot \bar{\mathcal{V}}^T, [], 3),
	\end{array}
	\end{equation}
	where $\bar{\mathcal{Y}}=\bar{\mathcal{U}} \cdot \bar{\mathcal{S}} \cdot \bar{\mathcal{V}}^T$;
	$
	\bar{\mathcal{S}}_{dw*}(r,r,k)= \max ( \bar{\mathcal{S}}(r,r,k)-\tau w_rw_k  ,0 );
	$, $R=TW(k)$
	$w_r =\left\{ \begin{array}{l}
	\begin{matrix}
	0,&		r\le R\\
	\end{matrix}\\
	\begin{matrix}
	1,&		r>R\\
	\end{matrix}\\
	\end{array} \right. $;
	$w_k = \frac{c_1}{\log ( \lVert \boldsymbol{\bar{\textbf{X}}}^{( k )} \rVert _F^2+\varepsilon ) } +c_2$.
\end{proof}
\begin{remark}
	In theorem \ref{DWSVT_theorem}, there are two key weights that need to be given in advance. 
	Frequency weights $w_k$ can be calculated by Eq. \eqref{eq_5}. 
	Truncation weight $w_r$ can be calculated by Algorithm 1. 
	It refers to the way multi-rank is calculated in reference \cite{PSTNN}. 
	In Algorithm 1, the ratio of the maximum value of the singular values is chosen, which can also be replaced by the ratio of the sum of the singular values. 
\end{remark}

\begin{algorithm}[htbp]\label{truncation_weight01}
		\caption{Calculate truncation weight $TW \in \mathbb{R}^{n_3}$. }
	{\bf Input:} A tensor $\mathcal{Y}$; ratio $\eta $. \\
	{\bf Output:}Truncation weight $TW \in \mathbb{R}^{n_3}$. \\
	\hspace*{0.02in} 1: Compute  $\overline{\mathcal{Y}}=\operatorname{fft}(\boldsymbol{Y},[], 3)$ . \\
	\hspace*{0.02in} 2: $\text{for}\  i=1, \cdots, n_3 \ \text{do} $ \\
	\hspace*{0.4in} $[\overline{\mathbf{U}}^{(i)}, \overline{\mathbf{S}}^{(i)}, \overline{\mathbf{V}}^{(i)} ]$  \\
	\hspace*{0.4in} $ s=\text{diag}(\overline{\mathbf{S}}^{(i)}) $
	  \\ 
	  \hspace*{0.4in} $ TW(i)= \text{length}(s(s>max(s)\times \eta)) $. 
	  \\ 
	 \hspace*{0.2in} $\text{end for} $
\end{algorithm}

\subsection{Multi-modal and Double-Weighted TNN}
%Additionally,  
TNN in \eqref{eq_3} only approximates the correlations connected by mode-3 Fourier transform in the spatial dimensions with the spectral dimension.
It lacks of flexibility for describing low-rankness in all modes of HSIs.
To connect the $p$-th mode with other two modes, we can define the double-weighted TNN for each mode-$p$ permutation of  HSIs, i.e., $\|\mathcal{X}_{p}\|_{dw*}, p=1,2,3.$
As $\|\mathcal{X}_p\|_{dw*}$ are different according to different modes,
we use the weighted average of double-weighted TNNs along all modes to approximate the tensor rank of HSIs.
Finally, the multi-modal and double-weighted TNN (MDWTNN) is proposed as follows:
\begin{equation}
\label{eq_8}
\begin{aligned}
  \|\mathcal{X}\|_{mdw*} :=\sum_{p=1}^{3}\alpha_p\|\mathcal{X}_{p}\|_{dw*}=\sum_{p=1}^{3} \sum_{k=1}^{n_p}\alpha_p w_k^p\|\bar{\textbf{X}}^{(k)}_{p}\|_{\textrm{PSSV}},
\end{aligned}
\end{equation}
 where $\bar{\mathcal{X}}_{p}=\texttt{fft}(\mathcal{X}_{p},[],3)$;
 $\bar{\textbf{X}}^{(k)}_{p}$ is the $k$-th frontal slice of $\bar{\mathcal{X}}_{p}$ and its assigned weight is $w_{k}^{p}$;
 $\alpha_p>0$ and $\sum_{p=1}^3\alpha_p=1$.
 Fig. \ref{fig:mdwtnnshiytu04} shows the schematic diagram of MDWTNN. 
\begin{figure*}
	\centering
	\includegraphics[width=0.95\linewidth]{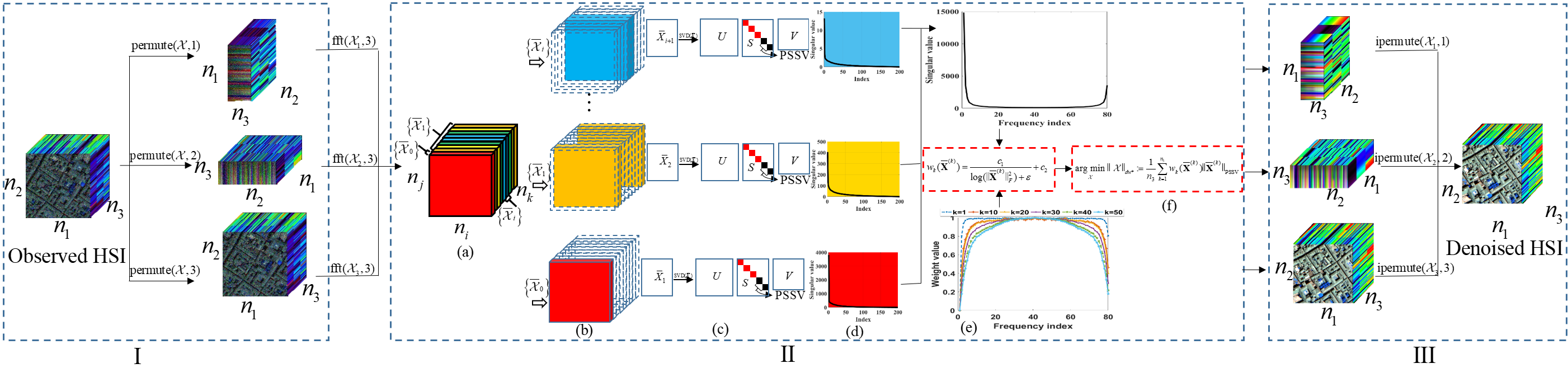}
	\caption{  MDWTNN Schematic diagram. (I) Multi-modal permutations and Fourier transforms.  (II) Double-weighted TNN. (a) mode-3 Fourier transform of $\mathcal{X}_p,p=1,2,3$. Due to the repeatability of the operation, II only represents one of $\mathcal{X}_1$, $\mathcal{X}_2$, $\mathcal{X}_3$, and the other two only need to repeat the operation represented by II.  (b) different frequency components. (c)  PSSV relaxation on each frequency slice $\{\bar{\mathcal{X}}_i\}$. (d) The singular values of  frequency slices. (e) Frequency weights. The top is the sum of singular values in each frequency slice, and the bottom is the weight of each frequency slice. (f) Doubel-weighted TNN minimization with respect to  $\mathcal{X}_p$. (III)  Synthesis from multi-modal denoised results. }   
	\label{fig:mdwtnnshiytu04}
\end{figure*}

\section{HSI denoising via MDWTNN Minimization}

MDWTNN in (\ref{eq_8}) takes full advantage of physical meanings in frequency components, singular values, and modes of HSIs,
which can provide a better approximation to the tensor rank.
Then we use MDWTNN to replace the regularization term $\texttt{Rank}$ in (\ref{eq_2}) and propose the HSI denoising model as follows: 
\begin{equation}
\label{MDWTNN_main}
%\begin{aligned}
\begin{array}{rl}
\arg\min_{\mathcal{X}, \mathcal{S}, \mathcal{N}} &\|\mathcal{X}\|_{mdw*}+ \lambda\|\mathcal{S}\|_1 + \tau\|\mathcal{N}\|_F^2, \\
s.t. & \mathcal{Y} = \mathcal{X} + \mathcal{S} + \mathcal{N}.
\end{array}
%\end{aligned}
\end{equation}

Introducing auxiliary variables, model \eqref{MDWTNN_main} is equivalent to
\begin{equation}
\label{MDWTNN_main2}
\begin{array}{rl}
\displaystyle \arg \min_{\mathcal{X},\mathcal{S},\mathcal{N}} \sum_{p=1}^{3}\alpha_p \|\mathcal{Z}_{p}\|_{dw*}+\lambda \lVert \mathcal{S} \rVert _1+\tau \lVert \mathcal{N} \rVert _{F}^{2},\\
\displaystyle  s.t.\ \mathcal{Y}=\mathcal{X}+\mathcal{S}+\mathcal{N},\ \mathcal{Z}_p=\mathcal{X}_p,\ p=1,2,3.\\
\end{array}
\end{equation}
By augmented Lagrangian multiplier method, the Lagrangian function of model \eqref{MDWTNN_main2} can be written as
\begin{equation}
\nonumber
\begin{array}{l}
L_{\mu _p,\beta}\left( \mathcal{X},\mathcal{Z}_p,\mathcal{N},\mathcal{S},\Gamma _p,\Lambda \right) = \lambda \lVert \mathcal{S} \rVert _1+\tau \lVert \mathcal{N} \rVert _{F}^{2} \\
+ <\mathcal{Y}-\left( \mathcal{X}+\mathcal{S}+\mathcal{N} \right) ,\Lambda > +\frac{\beta}{2}\lVert \mathcal{Y}-\left( \mathcal{X}+\mathcal{S}+\mathcal{N} \right) \rVert _{F}^{2} , \\
\displaystyle+\sum_{p=1}^3{\left\{ \alpha _p\lVert \mathcal{X}_p \rVert _{dw*}+<\mathcal{X}_p-\mathcal{Z}_p,\Gamma _p>+\frac{\mu _p}{2}\lVert \mathcal{X}_p-\mathcal{Z}_p \rVert _{F}^{2} \right\}},\\
\end{array}
\end{equation}
where $\Lambda$ and $\Gamma_p$ are the Lagrangian multipliers;
$\beta$ and $\mu _p$ are the Lagrange penalty parameters.
Its minimization problem can be efficiently solved in the framework of ADMM \cite{ADMM}.
At the $(n+1)$-th iteration, each variable in the Lagrangian function can be updated by solving its corresponding subproblem respectively when other variables are fixed at the $n$-th iteration.

For $\mathcal{Z}_p$, $p=1,2,3$, their corresponding subproblems can be written as
\begin{equation}
\label{update_Zp}
\arg\min_{\mathcal{Z}_{p}} \alpha _p\lVert \mathcal{Z}_p \rVert _{dw*}+\frac{\mu _p}{2}\left\| \mathcal{Z}_p-\left( \mathcal{X}_{p}^{n}+\frac{\Gamma _{p}^{n}}{\mu _p} \right) \right\| _{F}^{2}.
\end{equation}
%They can be solved by theorem 1 and their closed-form solutions are as follows :
The closed-form solution of \eqref{update_Zp} obtained from theorem \ref{DWSVT_theorem} are as follows:
\begin{equation}
\begin{array}{rl}
\mathcal{Z}_{p}^{n+1}=\mathcal{D}\mathcal{W}^{w\left( \mathcal{X}_p^n \right) ,R,\frac{\alpha _p}{\mu _p}}\left( \mathcal{X}_{p}^{n}+\frac{\Gamma _{p}^{n}}{\mu _p} \right). \\
\end{array}\label{updateZp}
\end{equation}

For  $ \mathcal{X} $, its corresponding subproblem can be written as
\begin{equation}
\begin{aligned}
\mathcal{X}^{n+1}=
&\arg\min_{\mathcal{X}} \sum_{p=1}^3{\frac{\mu _p}{2}}\lVert \mathcal{X}-\mathcal{Z}_{p}^{n+1}+\frac{\Gamma _{p}^{n}}{\mu _p} \rVert _{F}^{2}\\
&+\frac{\beta}{2}\lVert \mathcal{Y}-\left( \mathcal{X}+\mathcal{S}^n+\mathcal{N}^n \right) +\frac{\Lambda ^n}{\beta} \rVert _{F}^{2}.
\end{aligned}
\end{equation}
It has the closed-form solution as follows:
\begin{equation}
\begin{array}{rl}
\mathcal{X}^{n+1}=\frac{\sum_{p=1}^3{\mu _p}\left( \mathcal{Z}_{p}^{n+1}-\frac{\Gamma _{p}^{n}}{\mu _p} \right) +\beta \left( \mathcal{Y}-\mathcal{S}^n-\mathcal{N}^n+\frac{\Lambda ^n}{\beta} \right)}{1+\beta}. \\
\end{array}\label{update_X}
\end{equation}

For  $ \mathcal{S} $, its corresponding subproblem can be written as
\begin{equation}
\label{solve_S}
%\mathcal{S}^{n+1}=
\arg\min_{\mathcal{S}}\lambda \lVert \mathcal{S} \rVert _1+\frac{\beta}{2}\lVert \mathcal{Y}-\left( \mathcal{X}^{n+1}+\mathcal{S}+\mathcal{N}^{n} \right) +\frac{\Lambda ^n}{\beta} \rVert _{F}^{2}.
\end{equation}
It can be solved by the soft-thresholding operator \cite{L1_Solve} as:
\begin{equation}
\begin{array}{rl}
\mathcal{S}^{n+1}=\texttt{shrink}\left( \mathcal{Y}-\mathcal{X}^{n+1}-\mathcal{N}^{n+1}+\frac{\Lambda ^n}{\beta},\frac{\lambda}{\beta} \right).
\end{array}\label{update_S}
\end{equation}

For  $ \mathcal{N} $, its corresponding subproblem can be written as
\begin{equation}
%\mathcal{N}^{n+1}=
\arg\min_{\mathcal{N}} \tau \|\mathcal{N}\|_{F}^{2}+\frac{\beta}{2}\|\mathcal{Y}-\left( \mathcal{X}^{n+1}+\mathcal{S}^{n+1}+\mathcal{N} \right) +\frac{\Lambda ^n}{\beta}\|_{F}^{2}.
\end{equation}
It has the closed-form solution as follows : 
\begin{equation}
\begin{array}{rl}
\mathcal{N}^{n+1}=\frac{\beta \left( \mathcal{Y}-\mathcal{X}^{n+1}-\mathcal{S}^n+\frac{\Lambda ^n}{\beta} \right)}{2\tau +\beta}.
\end{array}\label{update_N}
\end{equation}

For multipliers $\Gamma_{p}$ and $\Lambda$, they can be updated as follows: 
\begin{equation}
\left\{ \begin{array}{l}
\Gamma _{p}^{n+1}=\Gamma _{p}^{n}+\mu _p\left( \mathcal{Z}_p^{n+1}-\mathcal{X}^{n+1} \right), p=1,2,3\\
\Lambda ^{n+1}=\Lambda ^n+\beta \left( \mathcal{Y}-\mathcal{X}^{n+1}-\mathcal{S}^{n+1}-\mathcal{N}^{n+1} \right). \\
\end{array} \right. \label{update_multiplier}
\end{equation}

The proposed algorithm for our HSI denoising model is summarized in Algorithm 2.
\begin{algorithm}[htbp]
	\caption{HSI denoising via the MDWTNN minimization}
	{\bf Input:} The observed tensor $\mathcal{Y}$; weight parameters $c_1$, $c_2$, ratio $\eta $;  regularization parameters $\lambda$, $\tau$; and  stopping criterion $\epsilon$. \\
	{\bf Output:} Denoised image $\mathcal{X}$. \\
	\hspace*{0.02in} 1: Initialize: $\mathcal{Y}$=$\mathcal{X}$=$\mathcal{S}$=$\mathcal{N}$=$\mathcal{Z}_p$; $\Gamma _p=\Lambda =0$;  $\mu_p$=$\beta $=$10^{-3}$; \\
	\hspace*{0.07in}$p=1,2,3$; $\mu_{max}=10^{10}$;  $\rho=1.2$ and $n=0$. \\
	\hspace*{0.02in} 2: Repeat until convergence:\\
%	\hspace*{0.02in} 3. Update $\mathcal{X}, \mathcal{S}, \mathcal{N}, \mathcal{Z}_p, \Lambda$, $\beta$, $\mu_p$, $w_k$, $\Gamma_p$ via\\
	\hspace*{0.02in} 3. Update $ \mathcal{Z}_p $,$ \mathcal{X} $,$ \mathcal{S} $,$ \mathcal{N} $,$ \mathcal{X} $ $ \Gamma _{p}, \Lambda $ by \eqref{updateZp}, \eqref{update_X}, \eqref{update_S}, \eqref{update_N}, \eqref{update_multiplier} \\
%	\hspace*{0.2in} step 2: Update $ \mathcal{X} $ by \eqref{update_X}\\
%	\hspace*{0.2in} step 3: Update $ \mathcal{S} $ by \eqref{update_S}\\
%	\hspace*{0.2in} step 4: Update $ \mathcal{N} $ by \eqref{update_N}\\
%	\hspace*{0.2in} step 5: Update $ \Gamma _{p}, \Lambda $ by \eqref{update_multiplier}\\	
	\hspace*{0.2in} Update $\mu_p=\rho\mu_p$, $\beta=\rho\beta$, $w_k$ by \eqref{eq_5}\\
	\hspace*{0.02in} 4: Check the convergence condition.
\end{algorithm}
\section{Experiments}
To verify the effectiveness of our MDWTNN based HSI denoising model, various experiments were performed on two challenging simulated datasets and two real HSI datasets.
For comparison, four state-of-the-art HSI denoising methodes were employed as the benchmark in the experiments, i.e., BM4D \cite{BM4D}, LRMR \cite{LRMR},  LRTDTV \cite{LRTDTV} and 3DTNN \cite{3DTNN}.
Since the BM4D method was only suitable to remove Gaussian noise, we implemented it on HSIs which were preprocessed by the RPCA denoising method \cite{lu2019TRPCA}.

%\vspace{-0.6cm}
\subsection{ Simulated Data Experiments } 

In the simulation experiments, we selected two datasets.
From the Washington DC Mall dataset\footnote{\url{http://lesun.weebly.com/hyperspectral-data-set.html}}, we chose a sub-block with the size of $256 \times 256\times 191$ as a simulation dataset.
From the Pavia City Center dataset\footnote{\url{http://www.ehu.eus/ccwintco/index.php/}}, we chose a sub-block with the size of $200 \times 200\times 80$ as a simulation dataset. 
The hybrids of white Gaussian and impulse noises with 5 different intensity levels were added to the simulation datasets band by band.
Let $G$ and $P$ denoted the variance of Gaussian white noise and percentage of impulse noise, respectively.
In noise case 1-3, the same intensity noise was added to all the bands.
In noise case 1, $G$=0.1 and $P$=0.2;
In noise case 2, $G$=0.2 and $P$=0.2;
In noise case 3, $G$=0.1 and $P$=0.4;
In noise case 4 and 5, the noise intensities were different for different bands.
In noise case 4, $G$ was randomly selected from 0.1 to 0.2 and $P$=0.2;
In noise case 5, $G$=0.1 and $P$ was randomly selected from 0.2 to 0.4.

For quantitatively evaluating the denoised results of all the test methods,
the CPU times and the means of PSNR, SSIM  and SAM in each band, i.e., MPSNR, MSSIM and MSAM, were listed in Table \ref{tab:model2}.
Although the CPU times of our model were not the shortest, one could update $\mathcal{Z}_p$ by \eqref{updateZp} in parallel to further shorten the CPU times of our model.
For visual evaluation in Fig. \ref{DCMall_imshowG01SPrand2040_80}, 
we showed the denoised results of the Pavia City Center dataset in Case 1.

\begin{table}[htbp]
	\centering
	\caption{Quantitative comparison and time of all competing methods under different levels of noises on simulate dataset. }
	\scalebox{0.68}{
		\begin{tabular}{ccccccccc}
			\toprule
			Dataset & Noise case & Index & Noise & BM4D  & LRMR  & LRTDTV & 3DTNN & Our \\
			\midrule
			\multirow{20}[10]{*}{\makecell[c]{Washington \\ DC Mall}} & \multirow{4}[2]{*}{Case1} & PSNR  & 11.068  & 31.014  & 31.567  & 32.800  & \underline{34.270 } & \textbf{36.095 } \\
			&       & SSIM  & 0.085  & 0.893  & 0.867  & 0.896  & \underline{0.936}  & \textbf{0.950 } \\
			&       & MSAM  & 43.139  & 4.576  & 5.042  & 4.327  & \underline{3.481 } & \textbf{2.937 } \\
			&       & time  & -  & 547.005  & 378.898  & 538.160  & 270.211  & 334.656  \\
			\cmidrule{2-9}          & \multirow{4}[2]{*}{Case2} & PSNR  & 10.216  & 27.192  & 27.642  & \underline{29.489}  & 29.055  & \textbf{32.525 } \\
			&       & SSIM  & 0.061  & 0.791  & 0.743  & \underline{0.807}  & 0.801  & \textbf{0.894 } \\
			&       & MSAM  & 45.297  & 6.830  & 7.859  & \underline{6.382}  & 6.726  & \textbf{4.393 } \\
			&       & time  & -  & 529.243  & 395.461  & 584.461  & 309.977  & 378.728  \\
			\cmidrule{2-9}          & \multirow{4}[2]{*}{Case3} & PSNR  & 8.305  & 29.691  & 29.183  & \underline{30.270 } & 29.572  & \textbf{34.185 } \\
			&       & SSIM  & 0.037  & \underline{0.866}  & 0.798  & 0.844  & 0.784  & \textbf{0.928 } \\
			&       & MSAM  & 50.092  & \underline{5.264}  & 6.579  & 6.182  & 6.542  & \textbf{3.604 } \\
			&       & time  & - & 528.440  & 394.272  & 582.318  & 316.547  & 375.719  \\
			\cmidrule{2-9}          & \multirow{4}[2]{*}{Case4} & PSNR  & 10.648  & 28.970  & 29.518  & 31.073  & \underline{31.852}  & \textbf{34.379 } \\
			&       & SSIM  & 0.073  & 0.848  & 0.810  & 0.859  & \underline{0.887}  & \textbf{0.930 } \\
			&       & MSAM  & 44.265  & 5.695  & 6.439  & 5.468  & \underline{4.873  }& \textbf{3.579 } \\
			&       & time  & -  & 541.960  & 396.478  & 537.374  & 273.466  & 340.113  \\
			\cmidrule{2-9}          & \multirow{4}[2]{*}{Case5} & PSNR  & 9.669  & 30.419  & 30.412  & 31.560  & \underline{32.783 } & \textbf{35.180 } \\
			&       & SSIM  & 0.060  & 0.883  & 0.836  & 0.874  & \underline{0.902 } & \textbf{0.942 } \\
			&       & MSAM  & 47.270  & 4.865  & 5.761  & 5.405  & \underline{4.311}  & \textbf{3.220 } \\
			&       & time  & - & 546.378  & 397.474  & 541.124  & 277.123  & 340.162  \\
			\midrule
			\multirow{20}[10]{*}{\makecell[c]{Pavia City \\ Center}} & \multirow{4}[2]{*}{Case1} & MPSNR & 11.122 & 29.701 & 31.259 & \underline{32.297} & 31.696 & \textbf{33.951} \\
			&       & MSSIM & 0.105 & 0.920  & 0.905 & 0.914 &\underline{0.924}  & \textbf{0.942} \\
			&       & MSAM  & 45.712 & 5.84  & 6.824 & 4.93  & \underline{4.844} & \textbf{4.426} \\
			&       & time  & -     & 112.235 & 113.723 & 131.91 & 54.461 & 72.224 \\
			\cmidrule{2-9}          & \multirow{4}[2]{*}{Case2} & MPSNR & 10.265 & 25.619 & 27.321 & \underline{28.605} & 27.009 & \textbf{30.124} \\
			&       & MSSIM & 0.074 & \underline{0.835} & 0.791 & 0.821 & 0.799 & \textbf{0.879} \\
			&       & MSAM  & 46.978 & 7.170  & 8.385 & \underline{6.685} & 6.923 & \textbf{5.666} \\
			&       & time  & -    & 113.246 & 113.115 & 129.155 & 56.610 & 71.322 \\
			\cmidrule{2-9}          & \multirow{4}[2]{*}{Case3} & MPSNR & 8.384 & 27.745 & 28.937 &\underline{29.741}  & 27.696 & \textbf{31.924} \\
			&       & MSSIM & 0.043 & \underline{0.892} & 0.848 & 0.873 & 0.790  & \textbf{0.916} \\
			&       & MSAM  & 48.580 & 6.745 & 7.74  & \underline{6.007 }& 9.875 & \textbf{5.095} \\
			&       & time  & -     & 115.908 & 116.961 & 127.981 & 54.448 & 71.435 \\
			\cmidrule{2-9}          & \multirow{4}[2]{*}{Case4} & MPSNR & 10.659 & 27.361 & 29.094 & \underline{30.251} & 28.934 & \textbf{32.041} \\
			&       & MSSIM & 0.088 & \underline{0.877} & 0.853 & 0.870  & 0.860  & \textbf{0.915} \\
			&       & MSAM  & 46.457 & 6.6   & 7.796 & 5.877 & \underline{6.389} & \textbf{5.048} \\
			&       & time  & -     & 115.684 & 133.241 & 129.088 & 51.669 & 68.251 \\
			\cmidrule{2-9}          & \multirow{4}[2]{*}{Case5} & MPSNR & 9.565 & 28.674 & 30.044 &\underline{30.977}  & 30.055 & \textbf{33.038} \\
			&       & MSSIM & 0.070  & \underline{0.907 }& 0.878 & 0.893 & 0.879 & \textbf{0.932} \\
			&       & MSAM  & 47.885 & 6.293 & 7.333 &\underline{5.493} & 6.795 & \textbf{4.682} \\
			&       & time  & -     & 133.811 & 81.989 & 120.418 & 61.742 & 77.520 \\
			\bottomrule
		\end{tabular}%
		\label{tab:model2}%
	}
\end{table}%

\begin{figure*}[htbp] \centering
	\setlength{\abovecaptionskip}{-0.03cm}
	\setlength{\belowcaptionskip}{-0.5cm}
	\subfloat[{\scriptsize Original image}]
	{
		\includegraphics[width=0.25\columnwidth]{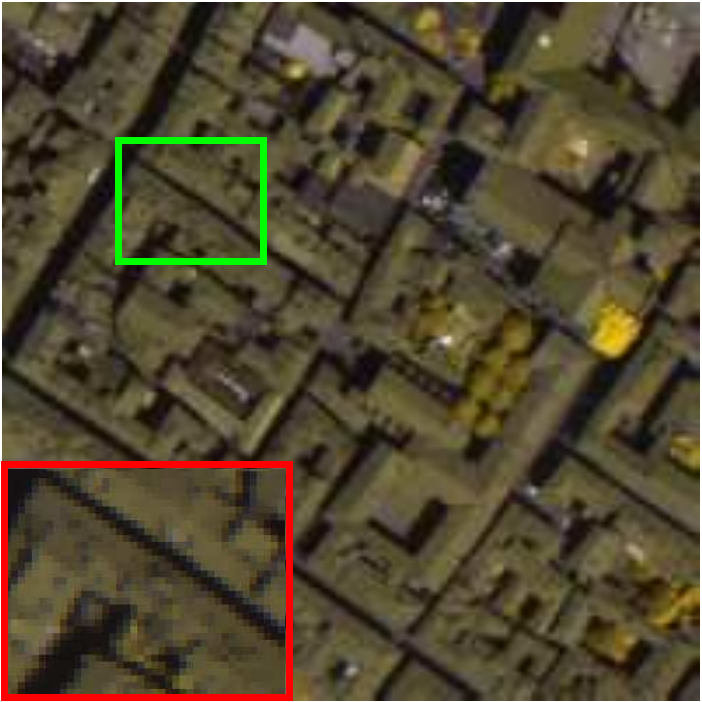}
	}\hfil
	\subfloat[{\scriptsize  Noise image}]
	{
		
		\includegraphics[width=0.25\columnwidth]{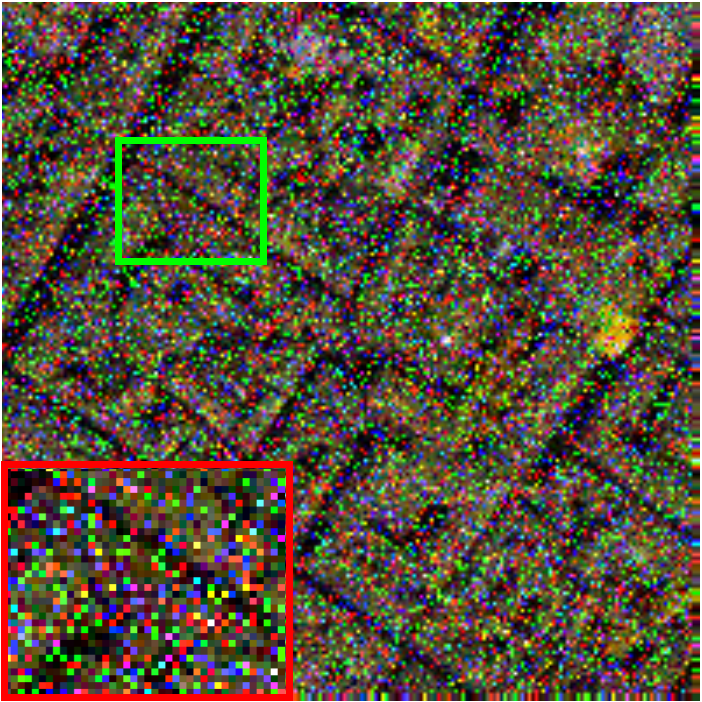}
	}\hfil
	\subfloat[{\scriptsize  BM4D(24.53dB)}]
	{
		\includegraphics[width=0.25\columnwidth]{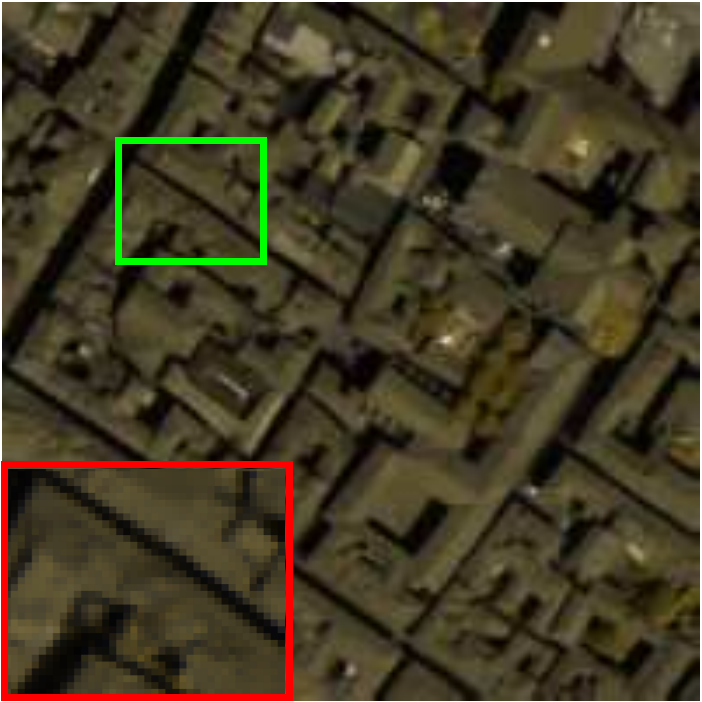}
	}\hfil
	\subfloat[{\scriptsize  LRMR(27.84dB)}]
	{
		
		\includegraphics[width=0.25\columnwidth]{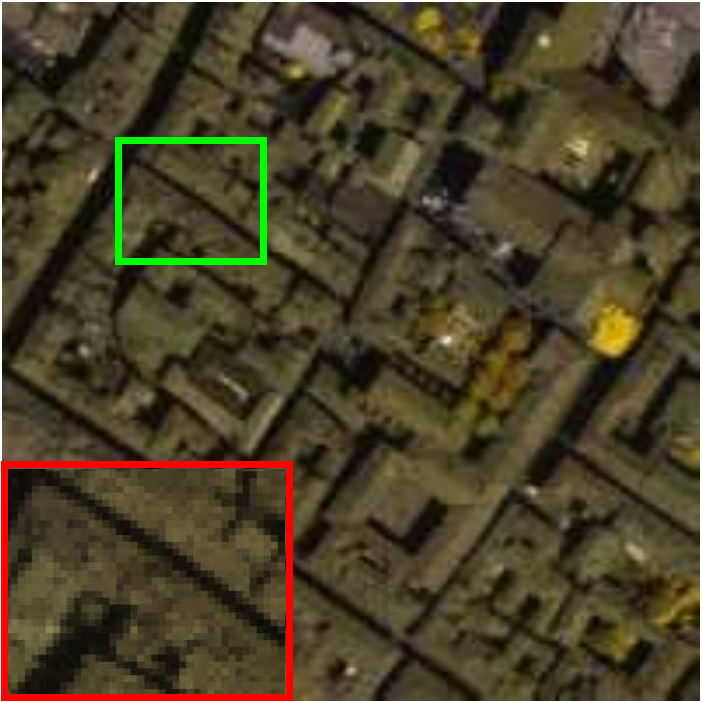}
	}\hfil
	\subfloat[  {\scriptsize  LRTDTV (27.13dB)} ]
	{
		
		\includegraphics[width=0.25\columnwidth]{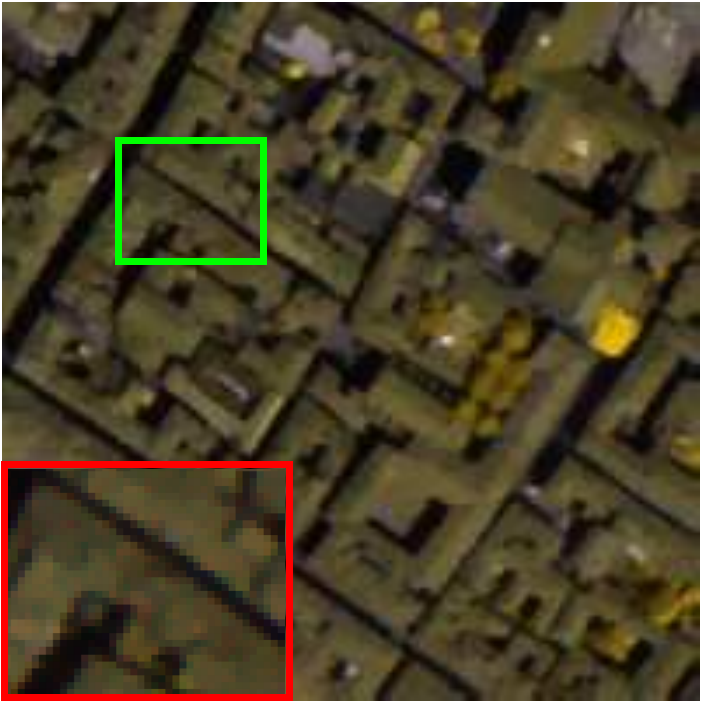}
	}\hfil
	\subfloat[{\scriptsize  3DTNN(27.55dB)}]
	{
		
		\includegraphics[width=0.25\columnwidth]{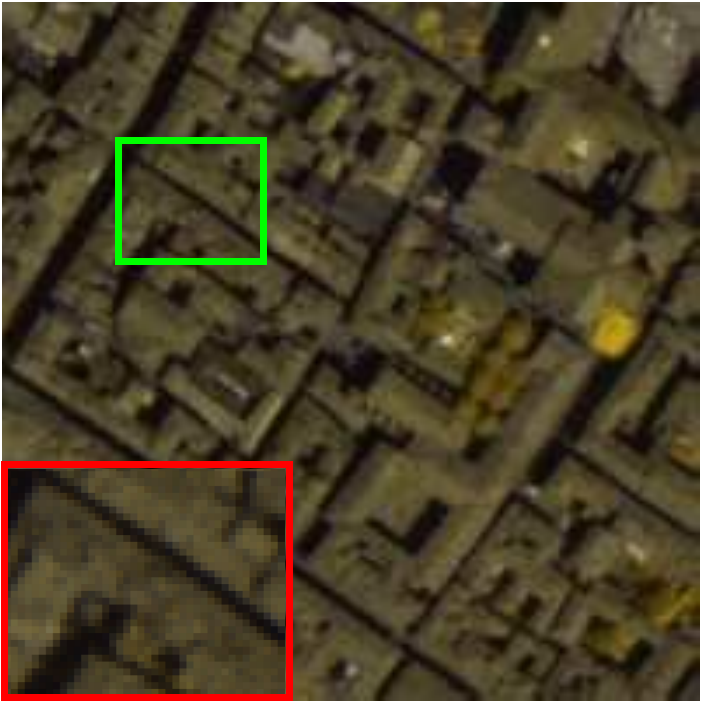}
	}\hfil
	\subfloat[{\scriptsize  Our(28.08dB)}]
	{
		
		\includegraphics[width=0.25\columnwidth]{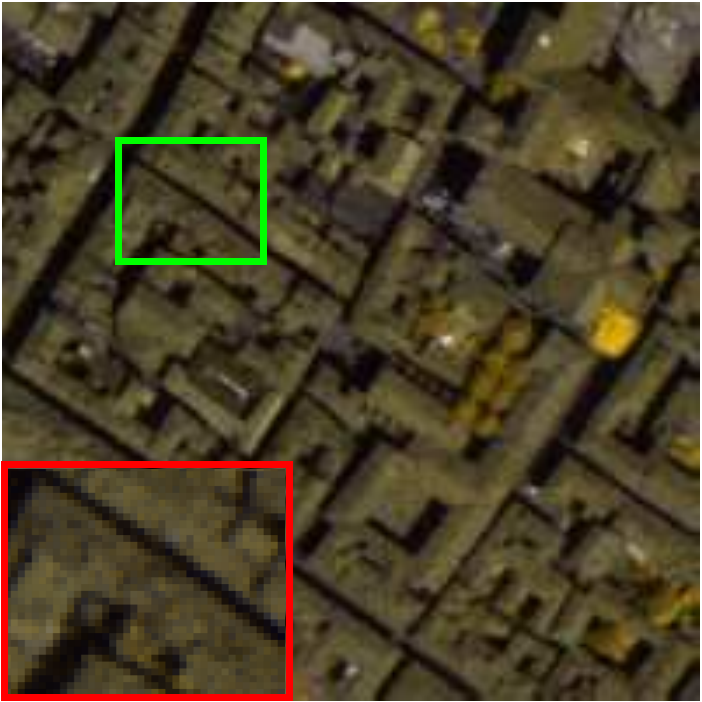}
	}\hfil
	\caption{The denoised results of the Pavia City Center dataset in Case 1.  Pseudocolor image with bands (78, 58, 14).} 
%	\caption{ The \textcolor{red}{denoised} results for the Washington DC Mall dataset under noise case 5. Pseudocolor image with bands (115, 80, 58). %The PSNRs are 10.51dB in (b), 27.78dB in (c), 28.87dB in (d), 29.44dB in (e), 30.84dB in (f) and 32.98dB in (g).
%	}
	\label{DCMall_imshowG01SPrand2040_80}
\end{figure*}

\subsection{  Real Data Experiments } 
 In the real experiments, we selected two datasets.
From the Indian Pines  dataset\footnote{\url{https://engineering.purdue.edu/∼biehl/MultiSpec/hyperspectral}}, we chose a sub-block with the size of $145 \times 145\times 224$.
From the Australian dataset\footnote{\url{http://remote-sensing.nci.org.au/u39/public/html/index.shtml}}, we chose a sub-block with the size of $200 \times 200\times 150$.
In Fig. \ref{real_Indian150_Australian48}, we listed all denoised results of the Indian Pines and Australian datasets.
For BM4D and LRTDTV, although they could remove more noise, they also lost more details. This made the denoised result too smooth.
For LRMR and 3DTNN, they mainly used the low-rank information of HSI. Although they retained more details, they also retained more noise.
Compared with them, our proposed model could remove more noise while retaining more details.
Fig. \ref{real_Indian218_DN} showed that the vertical mean profiles of band 218 before and after denoising. 
Here, one could see that the curve of the proposed MDWTNN method was most stable.  

\begin{figure*}[htbp] \centering
	\setlength{\abovecaptionskip}{-0.05cm}
	\subfloat[Real band] {
		
		\includegraphics[width=0.25\columnwidth]{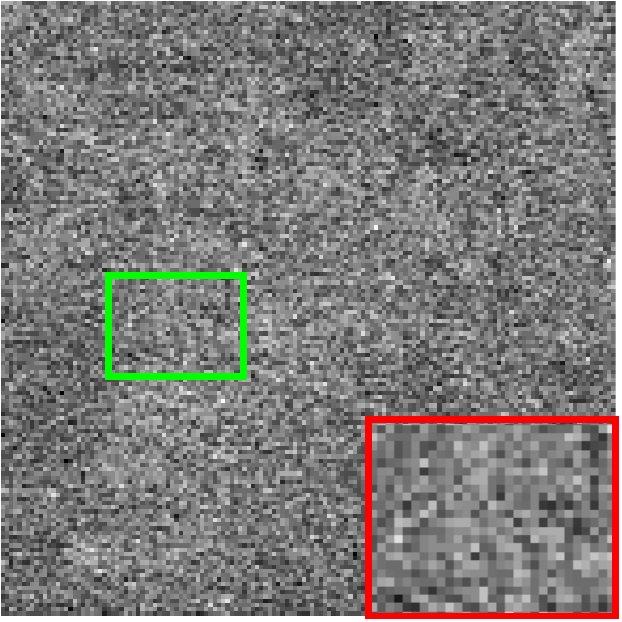}
	}\hfil
%	\subfloat[LRTA] {
%		
%		\includegraphics[width=0.25\columnwidth]{image/real_Indian/real_imshowIndian_150LRTA1}
%	}\hfil
	\subfloat[BM4D] {
		
		\includegraphics[width=0.25\columnwidth]{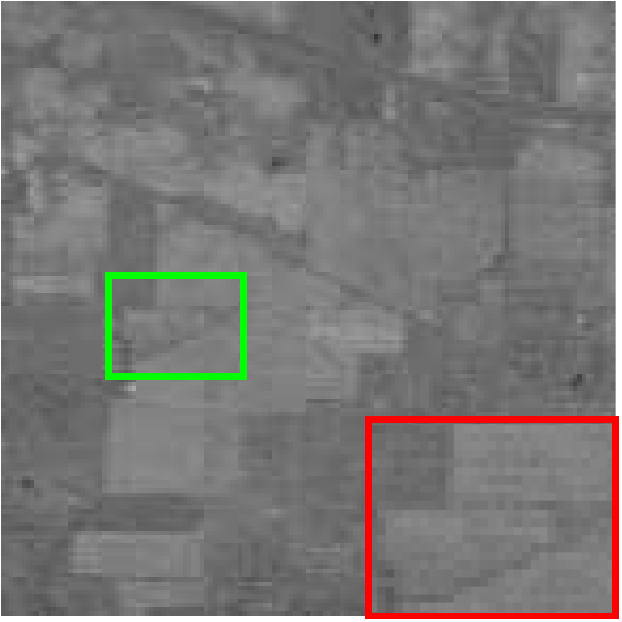}
	}\hfil
	\subfloat[LRMR] {
		
		\includegraphics[width=0.25\columnwidth]{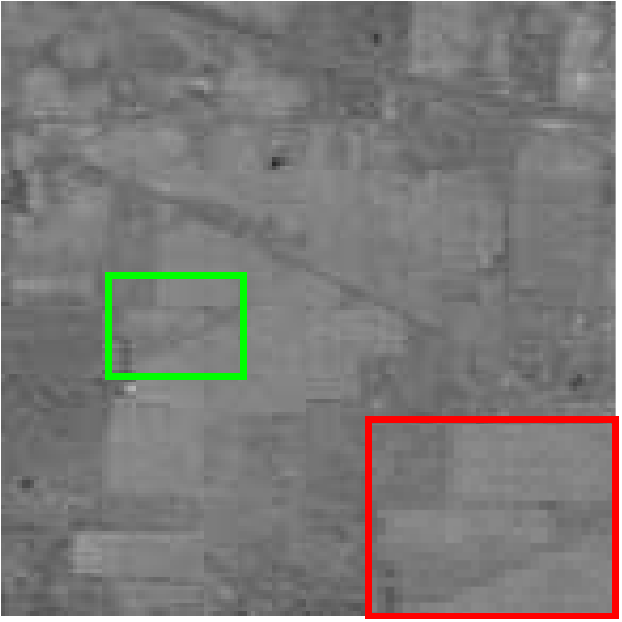}
	}\hfil
	\subfloat[LRTDTV] {
		
		\includegraphics[width=0.25\columnwidth]{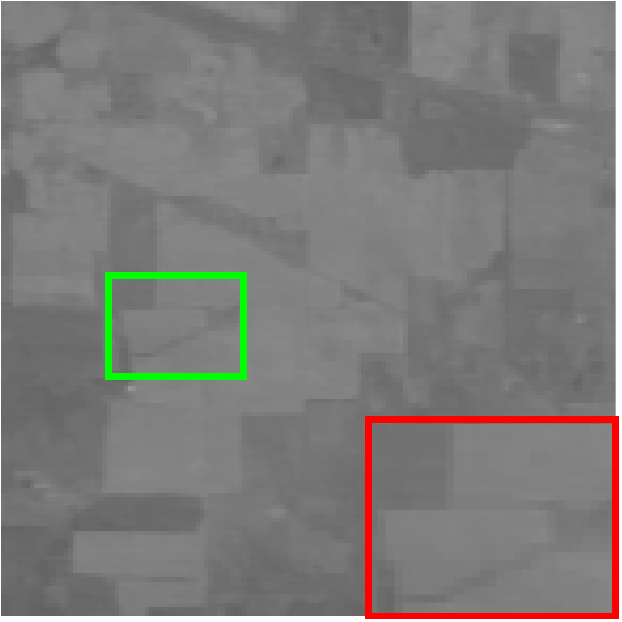}
	}\hfil
	\subfloat[3DTNN] {
		
		\includegraphics[width=0.25\columnwidth]{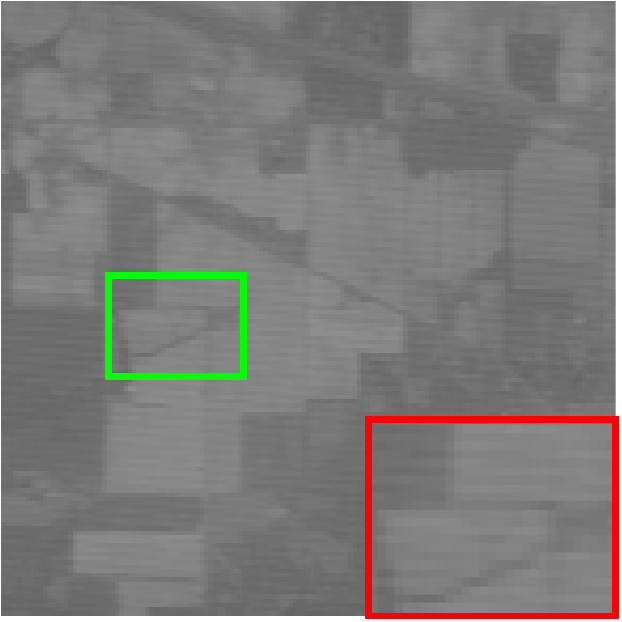}
	}\hfil
	\subfloat[Our] {
		
		\includegraphics[width=0.25\columnwidth]{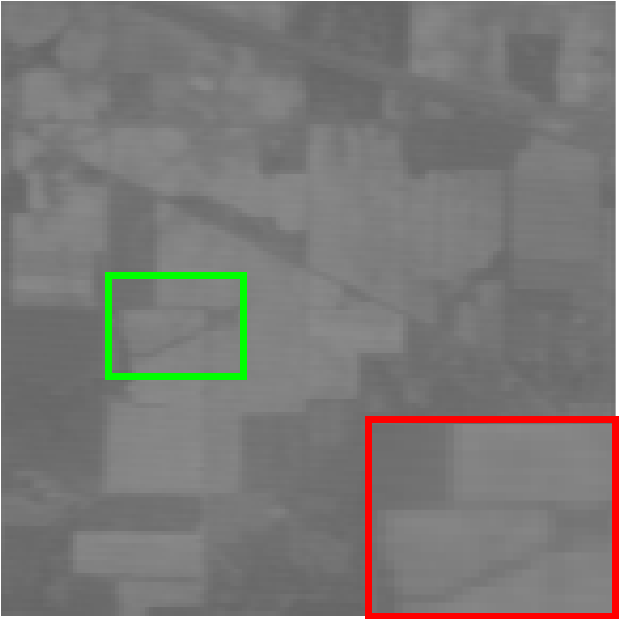}
	}\hfil \\
	\subfloat[Real band] {
	
	\includegraphics[width=0.25\columnwidth]{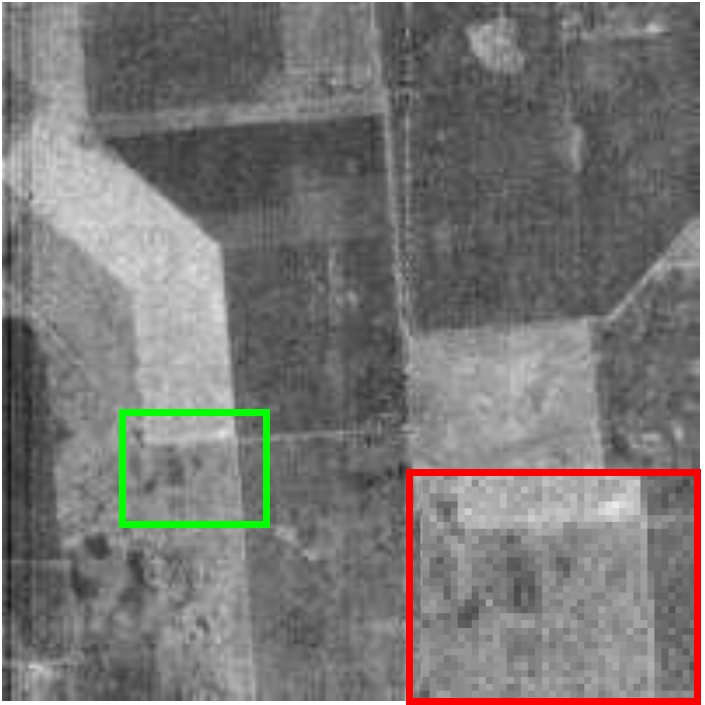}
	}\hfil
	%	\subfloat[LRTA] {
	%		
	%		\includegraphics[width=0.25\columnwidth]{image/real_Australiu_48/real_imshowAustralia_49LRTA}
	%	}\hfil
	\subfloat[BM4D] {
		
		\includegraphics[width=0.25\columnwidth]{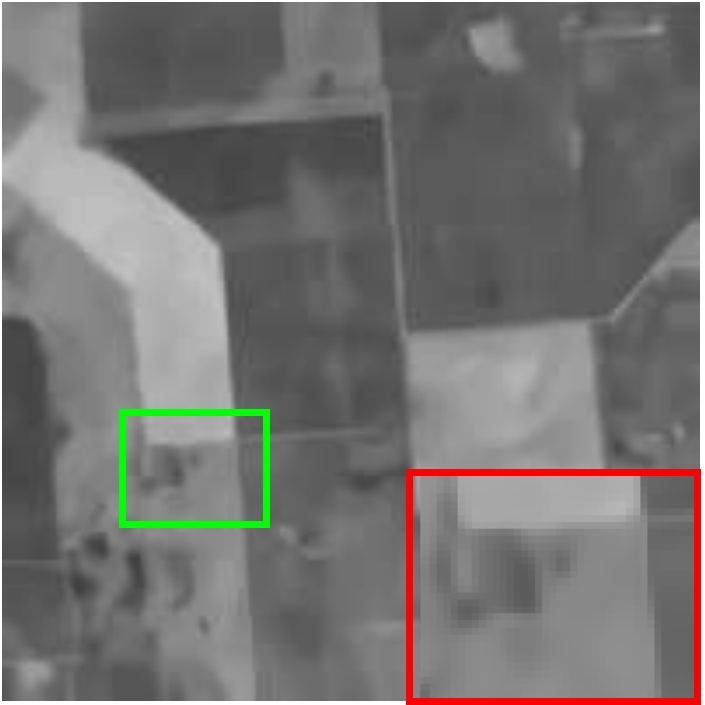}
	}\hfil
	\subfloat[LRMR] {
		
		\includegraphics[width=0.25\columnwidth]{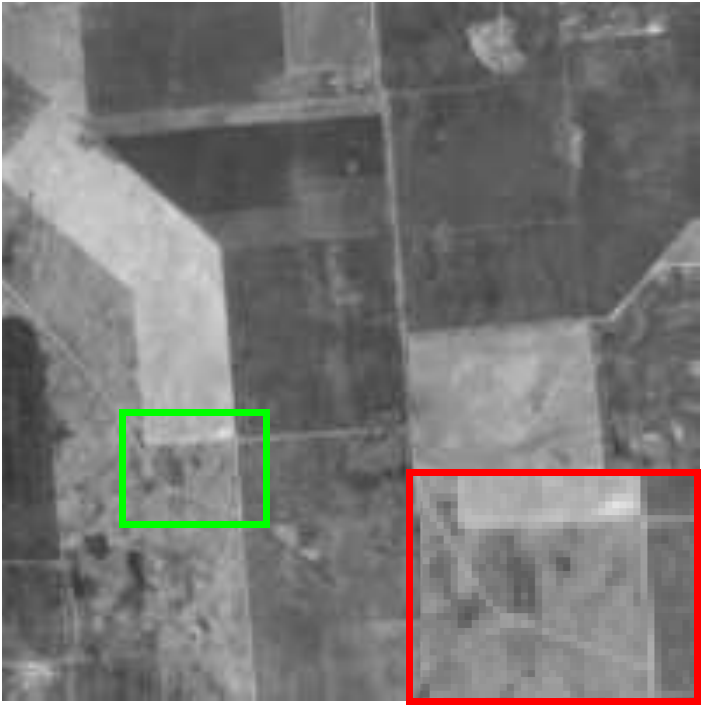}
	}\hfil
	\subfloat[LRTDTV] {
		
		\includegraphics[width=0.25\columnwidth]{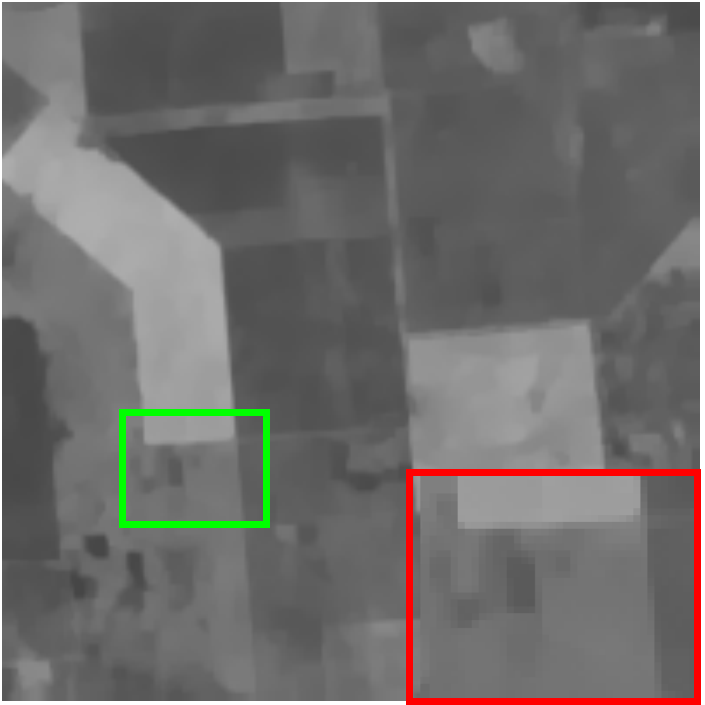}
	}\hfil
	\subfloat[3DTNN] {
		
		\includegraphics[width=0.25\columnwidth]{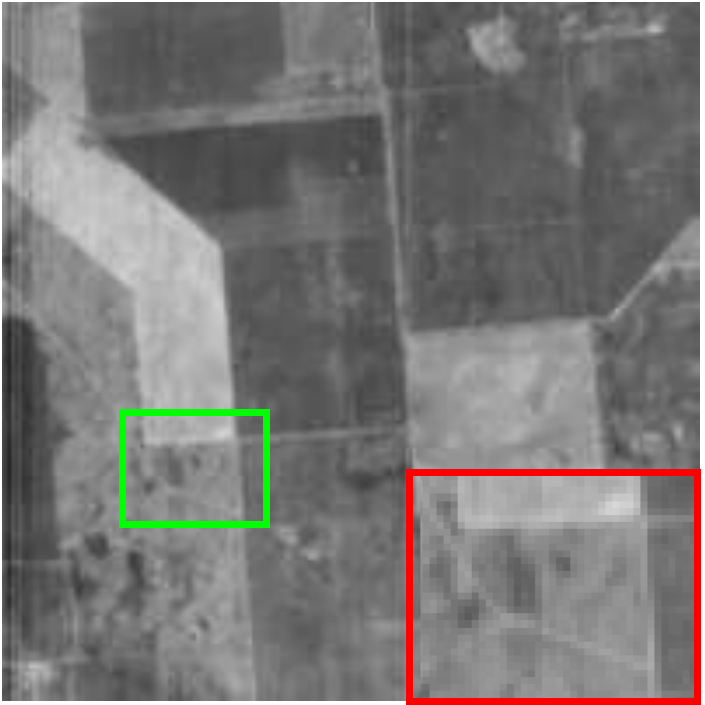}
	}\hfil
	\subfloat[Our] {
		
		\includegraphics[width=0.25\columnwidth]{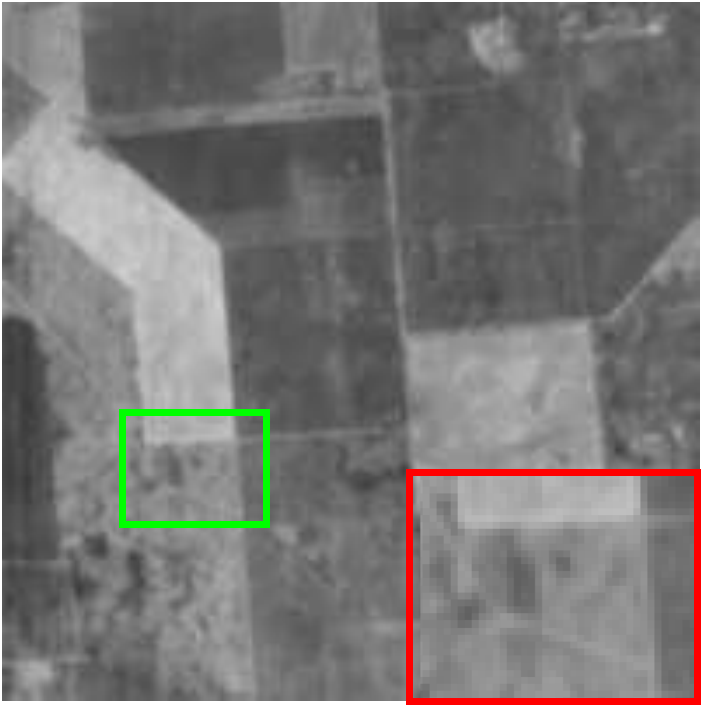}
	}\hfil
	\caption{  All denoised results for the Indian Pines and Australian dataset. (a)-(f) are the denoised results of the 150 band of the Indian Pines dataset. (g)-(l) are the denoised results of the 48 band of the Australian dataset.  } 
	\label{real_Indian150_Australian48}
\end{figure*}

\begin{figure*}[htbp] \centering
	\setlength{\abovecaptionskip}{-0.05cm}
		\setlength{\belowcaptionskip}{-0.5cm}
	\subfloat[Real dataset] {
		
		\includegraphics[width=0.25\columnwidth]{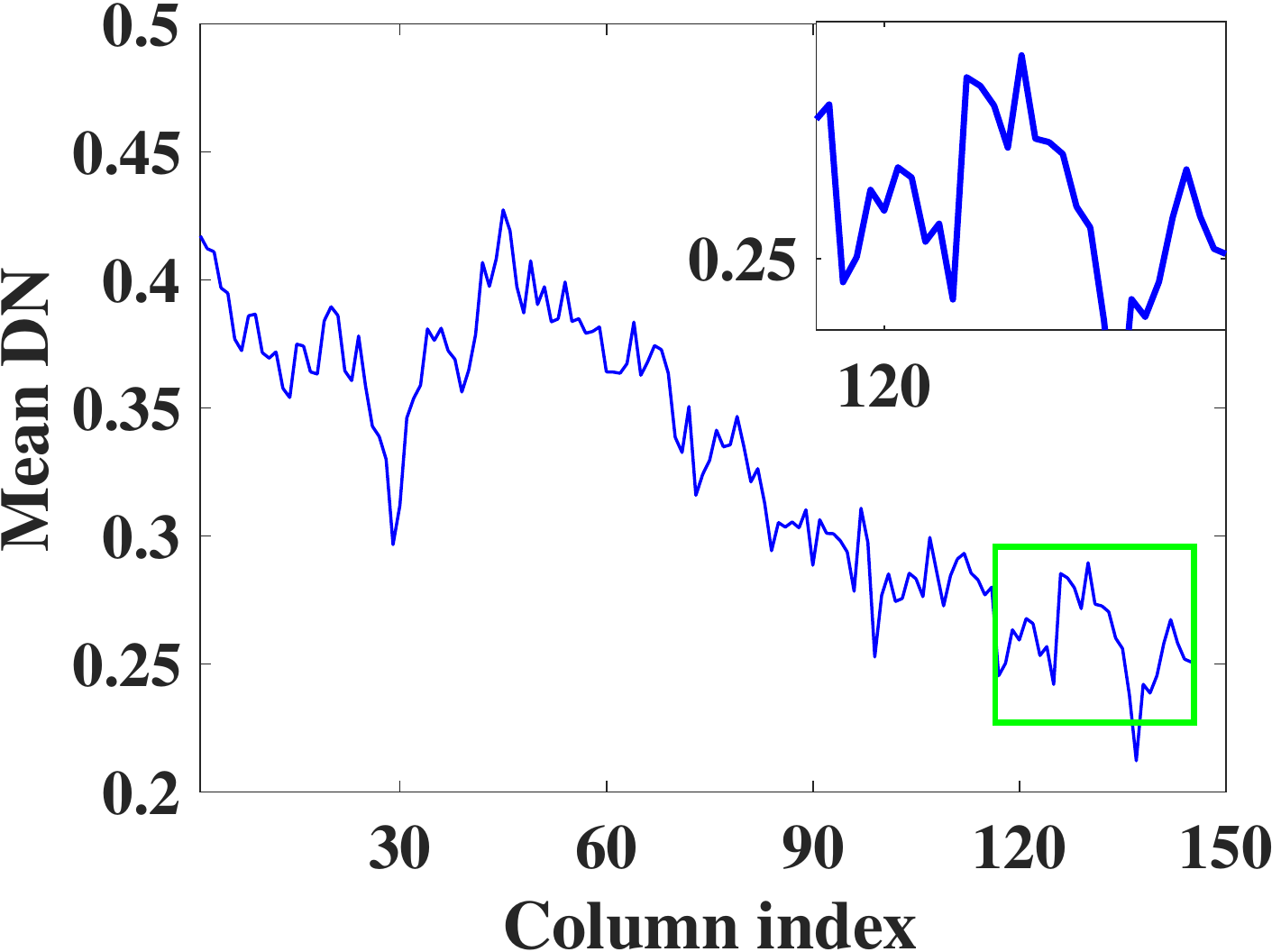}
	}\hfil
	%	\subfloat[LRTA] {
	%		
	%		\includegraphics[width=0.25\columnwidth]{image/real_Indian_DN218/real_imshowIndian_mean_DNLRTA}
	%	}\hfil
	\subfloat[BM4D] {
		
		\includegraphics[width=0.25\columnwidth]{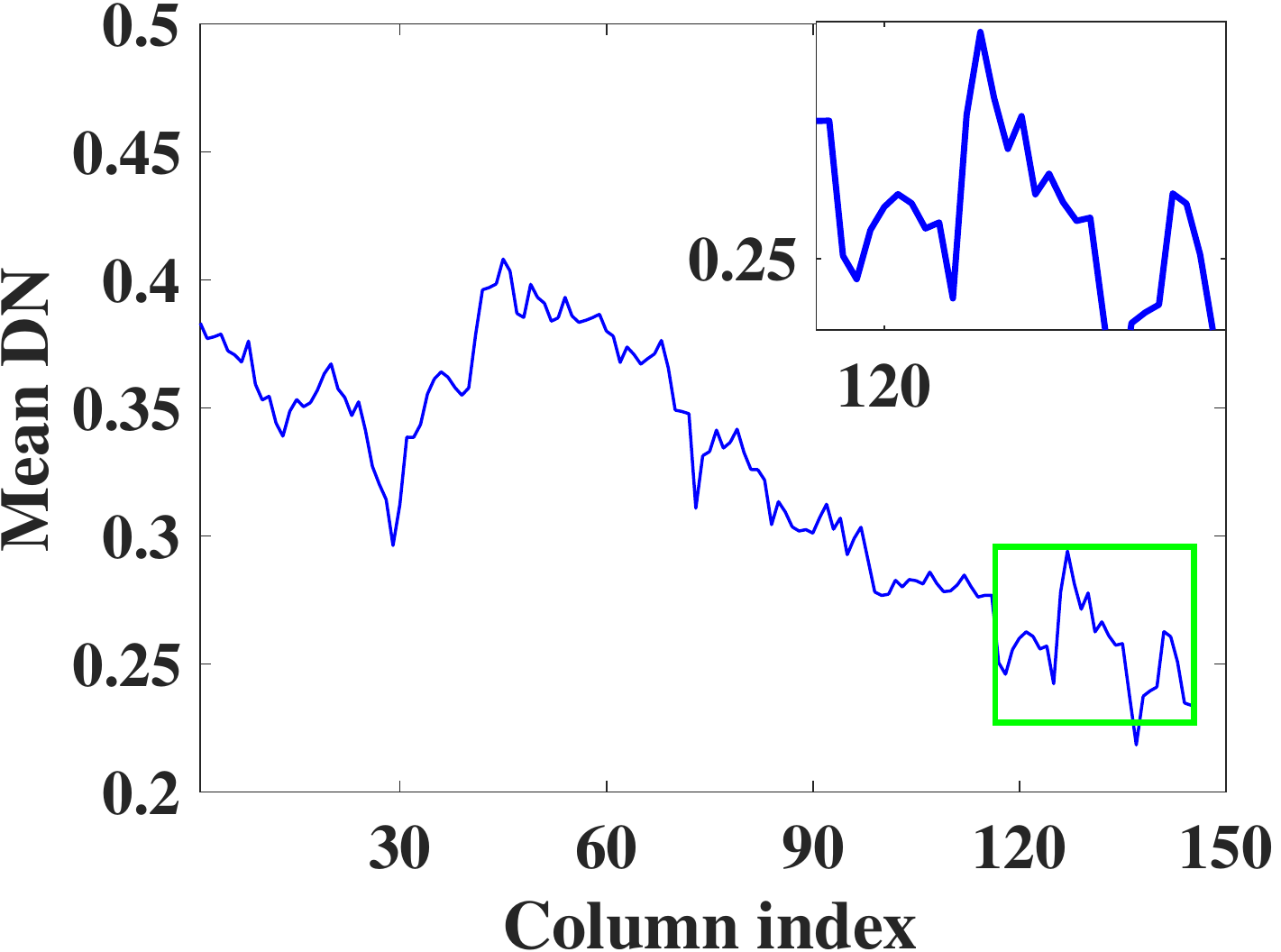}
	}\hfil
	\subfloat[LRMR] {
		
		\includegraphics[width=0.25\columnwidth]{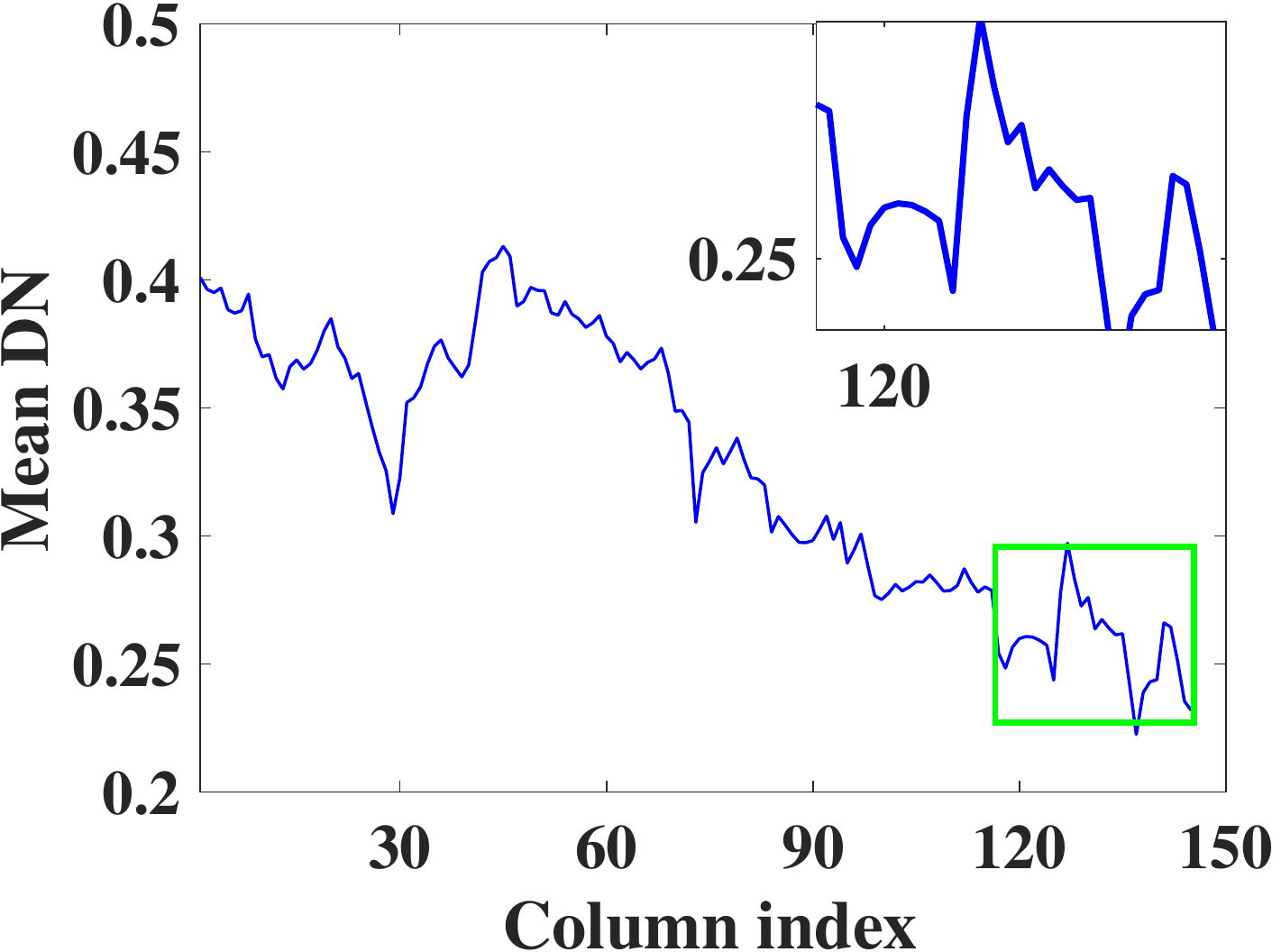}
	}\hfil
	\subfloat[LRTDTV] {
		
		\includegraphics[width=0.25\columnwidth]{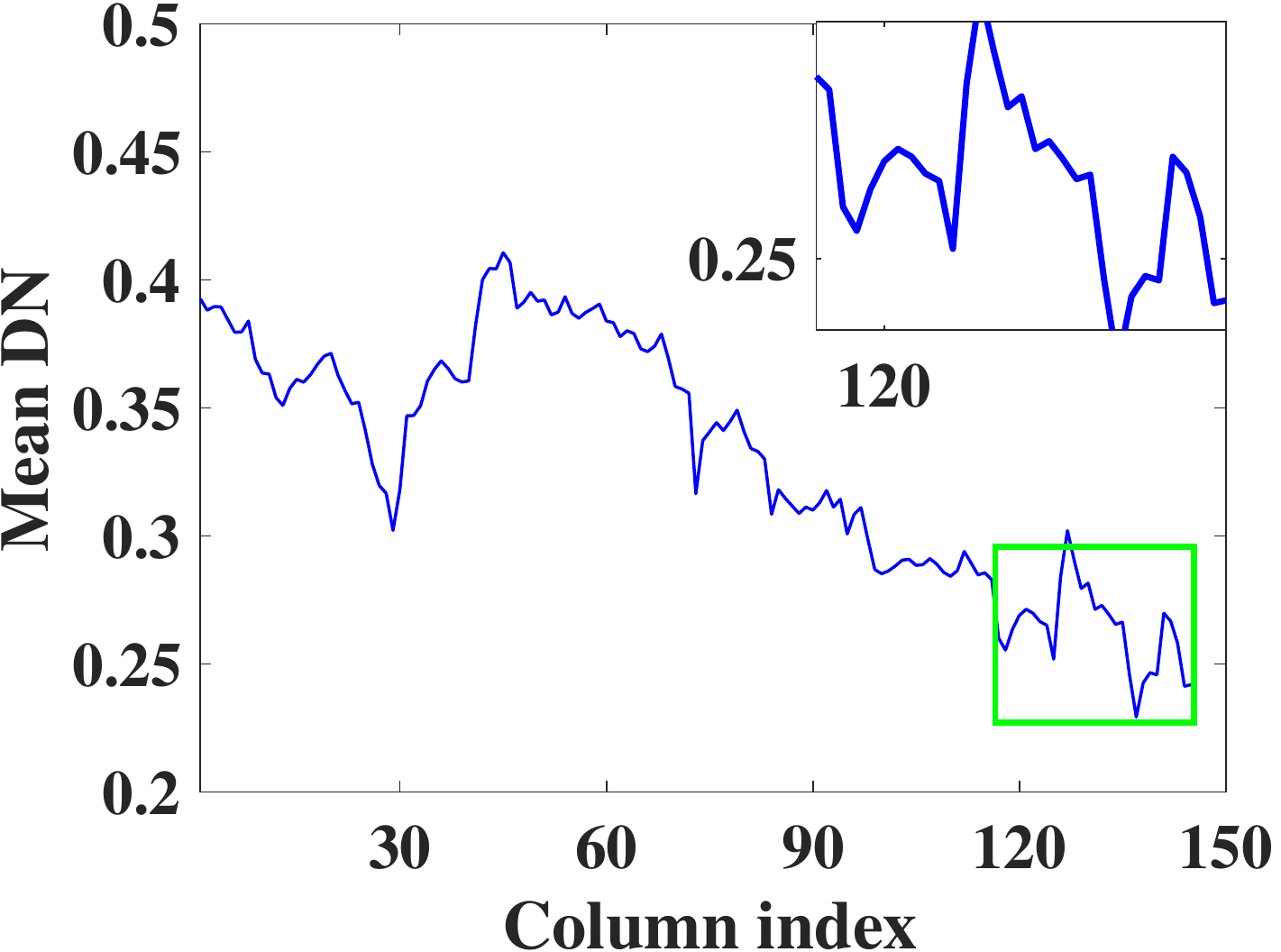}
	}\hfil
	\subfloat[3DTNN] {
		
		\includegraphics[width=0.25\columnwidth]{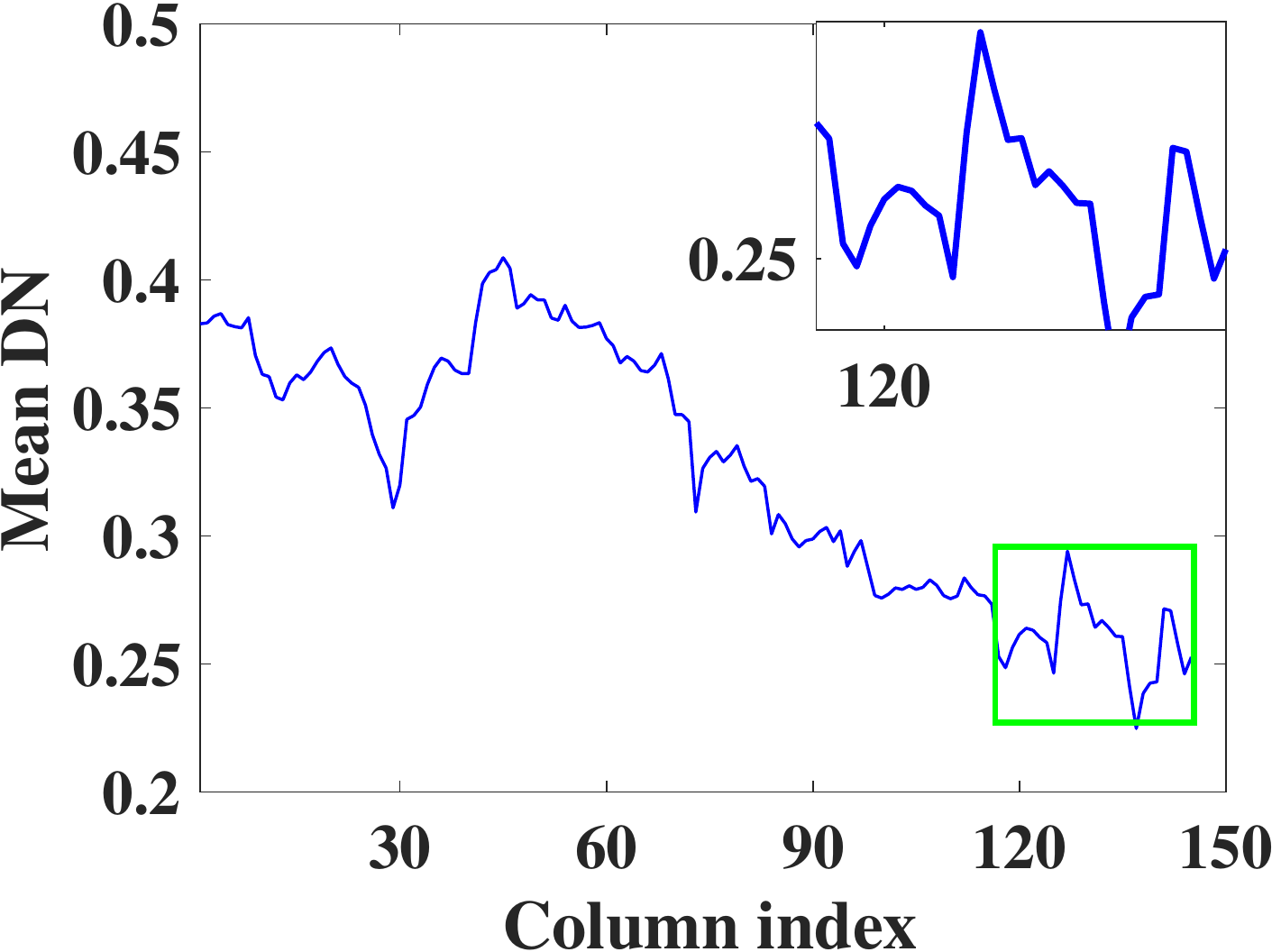}
	}\hfil
	\subfloat[Our] {
		
		\includegraphics[width=0.25\columnwidth]{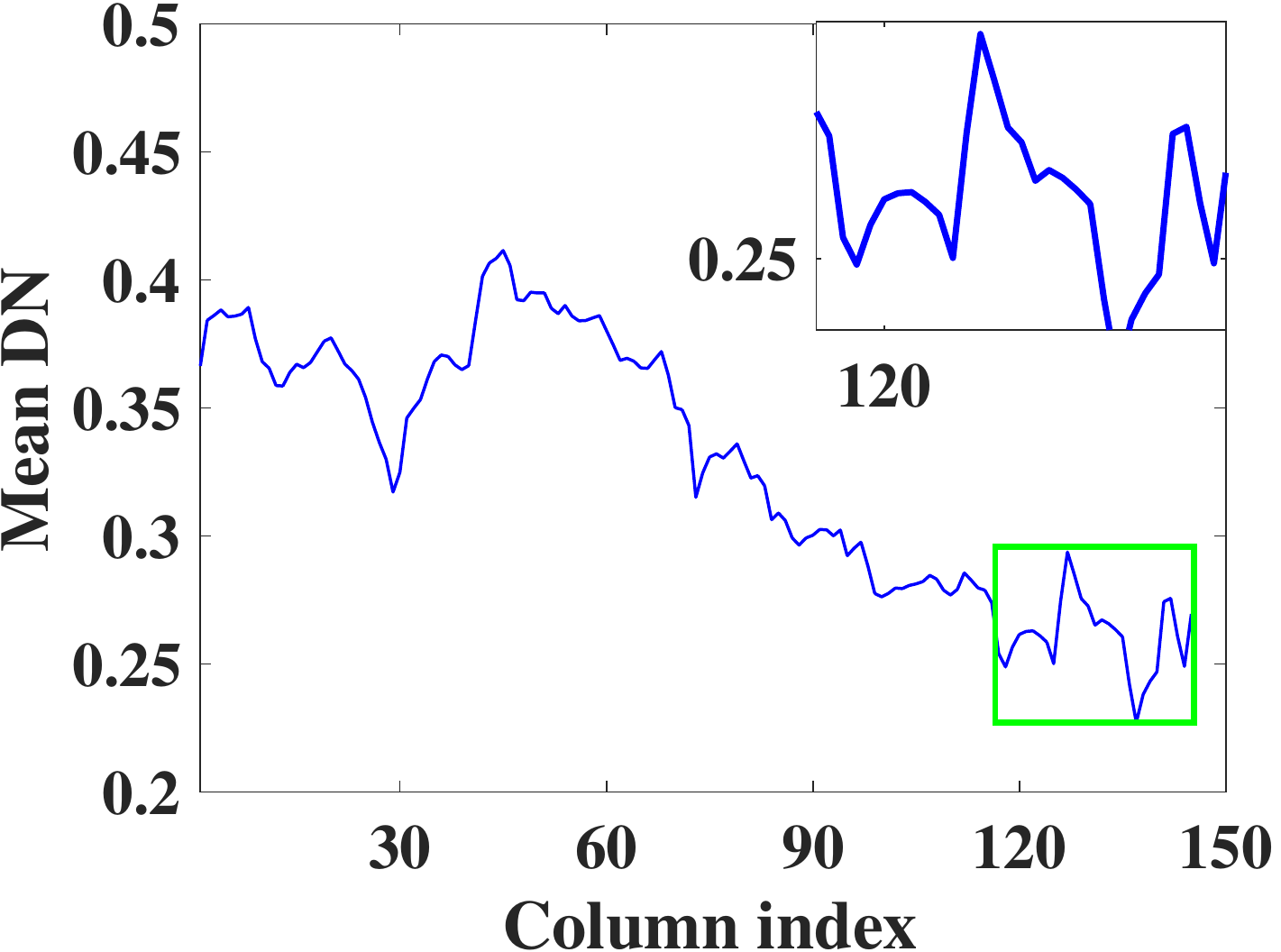}
	}\hfil
	\caption{  Vertical mean profiles of band 218 by all denoised results for the Indian Pines.  } 
	\label{real_Indian218_DN}
\end{figure*}

\section{Conclusion}

In this letter, we propose a multi-modal and double-weighted TNN for HSI denoising tasks.
The proposed TNN can efficiently characterize the physical meanings of the frequency components, singular values, and orientations ignored by the standard TNN.
And the weight parameters also can be obtained adaptively.
They powerfully improve capability and flexibility for describing low-rankness in HSIs.
%The proposed HSI \textcolor{red}{denoising} model can be efficiently solved in the framework ADMM.
The experiments conducted with both simulate and real HSI datasets show that
our MDWTNN based HSI denoising model is a competitive method to remove the hybrid noise.
Besides, our proposed MDWTNN regularization term can also be applied to other low-rankness based tasks, i.e., hyperspectral imagery classification, tensor completion, MRI reconstruction.

\small
\bibliographystyle{IEEEtran}% our
\bibliography{IEEEabrv,reference}% our

\begin{thebibliography}{10}
\providecommand{\url}[1]{#1}
\csname url@samestyle\endcsname
\providecommand{\newblock}{\relax}
\providecommand{\bibinfo}[2]{#2}
\providecommand{\BIBentrySTDinterwordspacing}{\spaceskip=0pt\relax}
\providecommand{\BIBentryALTinterwordstretchfactor}{4}
\providecommand{\BIBentryALTinterwordspacing}{\spaceskip=\fontdimen2\font plus
\BIBentryALTinterwordstretchfactor\fontdimen3\font minus
  \fontdimen4\font\relax}
\providecommand{\BIBforeignlanguage}[2]{{%
\expandafter\ifx\csname l@#1\endcsname\relax
\typeout{** WARNING: IEEEtran.bst: No hyphenation pattern has been}%
\typeout{** loaded for the language `#1'. Using the pattern for}%
\typeout{** the default language instead.}%
\else
\language=\csname l@#1\endcsname
\fi
#2}}
\providecommand{\BIBdecl}{\relax}
\BIBdecl

\bibitem{r2bioucas2013hyperspectral}
J.~M. Bioucas-Dias, A.~Plaza, G.~Camps-Valls, P.~Scheunders, N.~Nasrabadi, and
  J.~Chanussot, ``Hyperspectral remote sensing data analysis and future
  challenges,'' \emph{IEEE Geoscience and Remote Sensing Magazine}, vol.~1,
  no.~2, pp. 6--36, 2013.

\bibitem{Denoising_GRSL}
F.~Xu, Y.~Chen, C.~Peng, Y.~Wang, X.~Liu, and G.~He, ``Denoising of
  hyperspectral image using low-rank matrix factorization,'' \emph{IEEE
  Geoscience and Remote Sensing Letters}, vol.~14, no.~7, pp. 1141--1145, 2017.

\bibitem{LRTDTV}
Y.~Wang, J.~Peng, Q.~Zhao, Y.~Leung, X.-L. Zhao, and D.~Meng, ``Hyperspectral
  image restoration via total variation regularized low-rank tensor
  decomposition,'' \emph{IEEE Journal of Selected Topics in Applied Earth
  Observations and Remote Sensing}, vol.~11, no.~4, pp. 1227--1243, 2017.

\bibitem{Tuckear_GRSL}
S.~Meng, L.-T. Huang, and W.-Q. Wang, ``Tensor decomposition and pca jointed
  algorithm for hyperspectral image denoising,'' \emph{IEEE Geoscience and
  Remote Sensing Letters}, vol.~13, no.~7, pp. 897--901, 2016.

\bibitem{CP_r1}
X.~Liu, S.~Bourennane, and C.~Fossati, ``Denoising of hyperspectral images
  using the parafac model and statistical performance analysis,'' \emph{IEEE
  Transactions on Geoscience and Remote Sensing}, vol.~50, no.~10, pp.
  3717--3724, 2012.

\bibitem{lu2019TRPCA}
C.~Lu, J.~Feng, Y.~Chen, W.~Liu, Z.~Lin, and S.~Yan, ``Tensor robust principal
  component analysis with a new tensor nuclear norm,'' \emph{IEEE Transactions
  on Pattern Analysis and Machine Intelligence}, vol.~42, no.~4, pp. 925--938,
  2019.

\bibitem{3DTNN}
Y.-B. Zheng, T.-Z. Huang, X.-L. Zhao, T.-X. Jiang, T.-H. Ma, and T.-Y. Ji,
  ``Mixed noise removal in hyperspectral image via low-fibered-rank
  regularization,'' \emph{IEEE Transactions on Geoscience and Remote Sensing},
  vol.~58, no.~1, pp. 734--749, 2019.

\bibitem{t_SVD}
M.~E. Kilmer, K.~Braman, N.~Hao, and R.~C. Hoover, ``Third-order tensors as
  operators on matrices: A theoretical and computational framework with
  applications in imaging,'' \emph{SIAM Journal on Matrix Analysis and
  Applications}, vol.~34, no.~1, pp. 148--172, 2013.

\bibitem{zengTGRS}
H.~Zeng, X.~Xie, H.~Cui, H.~Yin, and J.~Ning, ``Hyperspectral image restoration
  via global $l_{1-2}$ spatial-spectral total variation regularized local
  low-rank tensor recovery,'' \emph{IEEE Transactions on Geoscience and Remote
  Sensing}, 2020.

\bibitem{zengSP}
H.~Zeng, X.~Xie, and J.~Ning, ``Hyperspectral image denoising via global
  spatial-spectral total variation regularized nonconvex local low-rank tensor
  approximation,'' \emph{Signal Processing}, vol. 178, p. 107805, 2021.

\bibitem{wang2020frequency}
S.~Wang, Y.~Liu, L.~Feng, and C.~Zhu, ``Frequency-weighted robust tensor
  principal component analysis,'' \emph{arXiv preprint arXiv:2004.10068}, 2020.

\bibitem{PSSV}
T.-H. Oh, Y.-W. Tai, J.-C. Bazin, H.~Kim, and I.~S. Kweon, ``Partial sum
  minimization of singular values in robust pca: Algorithm and applications,''
  \emph{IEEE Transactions on Pattern Analysis and Machine Intelligence},
  vol.~38, no.~4, pp. 744--758, 2015.

\bibitem{WNNM}
S.~Gu, L.~Zhang, W.~Zuo, and X.~Feng, ``Weighted nuclear norm minimization with
  application to image denoising,'' in \emph{Proceedings of the IEEE Conference
  on Computer Vision and Pattern Recognition}, 2014, pp. 2862--2869.

\bibitem{L1_Solve}
E.~T. Hale, W.~Yin, and Y.~Zhang, ``Fixed-point continuation for
  $\backslash$ell\_1-minimization: Methodology and convergence,'' \emph{SIAM
  Journal on Optimization}, vol.~19, no.~3, pp. 1107--1130, 2008.

\bibitem{PSTNN}
T.-X. Jiang, T.-Z. Huang, X.-L. Zhao, and L.-J. Deng, ``Multi-dimensional
  imaging data recovery via minimizing the partial sum of tubal nuclear norm,''
  \emph{Journal of Computational and Applied Mathematics}, vol. 372, p. 112680,
  2020.

\bibitem{ADMM}
S.~Boyd, N.~Parikh, and E.~Chu, \emph{Distributed optimization and statistical
  learning via the alternating direction method of multipliers}.\hskip 1em plus
  0.5em minus 0.4em\relax Now Publishers Inc, 2011.

\bibitem{BM4D}
M.~Maggioni and A.~Foi, ``Nonlocal transform-domain denoising of volumetric
  data with groupwise adaptive variance estimation,'' in \emph{Computational
  Imaging X}, vol. 8296.\hskip 1em plus 0.5em minus 0.4em\relax International
  Society for Optics and Photonics, 2012, p. 82960O.

\bibitem{LRMR}
H.~Zhang, W.~He, L.~Zhang, H.~Shen, and Q.~Yuan, ``Hyperspectral image
  restoration using low-rank matrix recovery,'' \emph{IEEE Transactions on
  Geoscience and Remote Sensing}, vol.~52, no.~8, pp. 4729--4743, 2013.

\end{thebibliography}

\end{document}